\documentclass[11pt]{article}
\pdfoutput=1

\usepackage{amsmath,amsbsy,amsfonts,amssymb,amsthm,dsfont,fullpage,units}
\usepackage{graphicx}
\usepackage{authblk}
\usepackage{algorithm,algorithmic,mathtools}
\usepackage{color,cases}
\usepackage{tikz}
\usetikzlibrary{calc,shapes}
\usepackage{subfigure}
\usepackage{hyperref}
\usepackage{bm}
\usepackage[algo2e, linesnumbered,ruled,vlined,boxed]{algorithm2e}

\newtheorem{claim}{Claim}

\newtheorem*{assumption*}{Assumption}
\newtheorem{lemma}{Lemma}
\newtheorem{conjecture}{Conjecture}

\newtheorem{theorem}{Theorem}[section]
\theoremstyle{definition}
\newtheorem{condition}{Condition}[section]
\newtheorem{definition}{Definition}

\newtheorem{remark}{Remark}
 \newcommand{\IGNORE}[1]{}

\newcommand{\citep}{\cite}
\newcommand{\citet}{\cite}

\newcommand\E{\mathbb{E}}
\newcommand\R{\mathbb{R}}
\newcommand\N{\mathcal{N}}

\newcommand\inner[1]{\langle #1 \rangle}

\newcommand{\Ynote}[1]{{\color{red}[Yuan: #1]}}

\title{Homotopy Analysis for Tensor PCA}
\author{
    Anima Anandkumar\thanks{University of California, Irvine. Email:\href{mailto:a.anandkumar@uci.edu}{a.anandkumar@uci.edu}} 
    \quad Yuan Deng\thanks{Duke University. Email:\href{mailto:ericdy@cs.duke.edu}{ericdy@cs.duke.edu}}
    \quad Rong Ge\thanks{Duke University. Email:\href{mailto:rongge@cs.duke.edu}{rongge@cs.duke.edu}}
    \quad Hossein Mobahi\thanks{Google Research. Email:\href{mailto:hmobahi@csail.mit.edu}{hmobahi@csail.mit.edu}}
}
\begin{document}
\maketitle

\begin{abstract}
Developing efficient and guaranteed nonconvex algorithms has been an important challenge in modern machine learning. Algorithms with good empirical performance such as stochastic gradient descent often lack theoretical guarantees. In this paper, we analyze the class of homotopy or continuation methods for global optimization of nonconvex functions. These  methods start from an objective function that is efficient to optimize (e.g. convex), and progressively modify it to obtain the required objective, and the solutions are passed along the homotopy path. For the challenging problem of tensor PCA, we prove global convergence of the homotopy method  in the ``high noise'' regime.   The signal-to-noise requirement for our algorithm is tight in the sense that it matches the recovery guarantee for the {\em best} degree-$4$ sum-of-squares algorithm. In addition, we prove a phase transition along the homotopy path for  tensor PCA. This allows us to simplify the homotopy method to a local search algorithm, viz., tensor power iterations, with a specific initialization and a noise injection procedure, while retaining the theoretical guarantees.\end{abstract}

\section{Introduction}

Non-convex optimization is a critical component in modern machine learning. Unfortunately, theoretical guarantees for nonconvex optimization have been mostly negative, and the problems are computationally hard in the worst case. Nevertheless, simple local-search algorithms such as stochastic gradient descent have enjoyed great empirical success in areas such as deep learning. As such, recent research efforts have attempted to bridge this gap between theory and practice.

For example, one property that can guarantee the success of local search methods over nonconvex functions is when all local minima are also the global minima. Interestingly, it has been recently proven that many well known nonconvex problems do have this property, under mild conditions. Consequently, local-search methods, which are designed to find a local optimum, automatically achieve global optimality. Examples of such problems include matrix completion~\citep{ge2016matrix}, orthogonal tensor decomposition~\citep{anandkumar2014tensor,ge2015escaping}, phase retrieval~\citep{DBLP:journals/corr/SunQW16}, complete dictionary learning~\citep{sun2015complete},  and so on. However, such a class of nonconvex problems is limited, and there  are many practical problems with poor local optima, where local search methods can fail.

The above property, while very helpful, imposes a strong assumption on the nonconvex problem. A less restrictive requirement for the success of local search methods is the ability to  initialize local search in the basin of attraction of the global optimum using another polynomial-time algorithm. This approach does not require all the local optima to be of good quality, and thus can cover a broader set of problems. Efficient initialization strategies have recently been developed for many nonconvex problems such as  overcomplete dictionary learning~\citep{arora2014new,agarwal2014learning}, tensor decomposition~\citep{anandkumar2015learning}, robust PCA~\citep{netrapalli2014non}, mixed linear regression~\citep{DBLP:journals/corr/YiCS16} and so on.

Although the list of such tractable nonconvex problems is growing, currently, the initialization algorithms are problem-specific and as such, cannot be directly extended to new problems. An interesting question is whether there exist common principles that can be used in designing efficient initialization schemes for local search methods. In this paper, we demonstrate how a class of homotopy continuation methods may provide such a framework for efficient initialization of local search schemes.

\subsection{Homotopy Method}

The homotopy method is a general and a problem independent technique for tackling nonconvex problems. It starts from an objective function that is efficient to optimize (e.g. convex function), and progressively transforms it to the required objective~\citep{MobahiLink}. Throughout this progression, the solution of each intermediate objective is used to initialize a local search on the next one.  A particular approach for constructing this progression is to smooth the objective function. Precisely, the objective function is convolved with the Gaussian kernel and the amount of smoothing is varied to obtain the set of transformations. Intuitively, smoothing ``erases wiggles'' on the objective surface (which can lead to poor local optima), thereby resulting in a function that is easier to optimize. %This allows the optimization algorithm to reach  near the global optima of the original objective function.
Global optimality guarantees for the homotopy method have been recently established~\citep{HosseinAAAI, HazanICML16}. However, the assumptions in these results are either too restrictive~\citep{HosseinAAAI} or extremely  difficult to check~\citep{HazanICML16}. In addition, homotopy algorithms are generally slow since local search is repeated within each instantiation of the smoothed objective.

%
% Since the homotopy paragraph is long enough here, I suggest moving the following citations to the homotopy section.
%
%  Homotopy methods have enjoyed empirical success in a number of areas such as optimization and differential equations~\cite{}, ansd have also recently been employed in  deep learning  \cite{MobahiRNN, Caglar}

In this paper, we address all the  above issues for the nonconvex tensor PCA problem. We analyze the homotopy method and guarantee convergence to global optimum under a set of transparent conditions. Additionally, we demonstrate how the homotopy method can  be drastically simplified without sacrificing the theoretical guarantees.  Specifically, by taking advantage of the phase transitions in the homotopy path, we can avoid the intermediate transformations of the objective function. In fact, we can start from the extreme case of ``easy'' (convex) function  of the homotopy, and use its solution to initialize local search on the original objective.  Thus, we show that the homotopy method can serve as a problem independent principle for obtaining a smart initialization which is then employed in local search methods.  Although we limit ourselves to the problem of tensor PCA in this paper, we expect the developed techniques to be applicable for a broader set of nonconvex problems.

\subsection{Tensor PCA}

Tensor PCA problem is an extension of the matrix PCA. The statistical model for tensor PCA was first introduced by~\cite{richard2014statistical}. This is a single spike model where the input tensor $ \bm T \in \R^{n\times n\times n}$ is a combination of an unknown rank-$1$ tensor and a Gaussian noise tensor $\bm A$ with $\bm A_{ijk}\sim \N(0,1)$  for $i,j,k\in [n].$   \begin{equation} \bm T = \tau \bm v\otimes \bm v \otimes \bm v + \bm A, \end{equation}    where $\bm v\in \R^n$ is the signal that we would like to recover.

Tensor PCA belongs to the class of ``needle in a haystack'' or high dimensional denoising problems, where the goal is to separate the unknown signal from a large amount of random noise. Recovery in the high noise regime has intimate connections to computational hardness, and has been extensively studied in a number of settings. For instance, in the spiked random matrix model, the input is an additive combination of an unknown rank-$1$ matrix and a random noise matrix. The requirement on the signal-to-noise ratio for simple algorithms, such as principal component analysis (PCA), to recover the unknown signal has been studied under various noise models~\citep{perry2016optimality,bloemendal2013limits} and sparsity assumptions on the signal vector~\citep{berthet2013optimal}.

%We consider the problem of tensor PCA where the goal is to find the best rank-$1$ approximation of a given tensor. We assume the single spike model where the input tensor is an additive combination of an unknown rank-one tensor and a Gaussian noise tensor. We propose a novel algorithm, and prove recovery of the rank-$1$ component under the best known signal-to-noise ratio for any polynomial-time algorithm. This guarantee was previously achieved by an expensive sum of squares algorithm while  our algorithm consists of efficient tensor power iterations with a specific initialization and a noise injection procedure.   Our algorithm is inspired by homotopy or continuation methods for global optimization. A main step in our proof is establishing a 

Previous algorithms for tensor PCA belong to two classes: local search methods such as tensor power iterations~\citep{richard2014statistical}, and global methods such as sum of squares~\citep{hopkins2015tensor}. Currently, the best signal-to-noise guarantee is achieved by the sum-of-squares algorithm and  the flattening algorithm, which are more expensive   compared to power iterations (see Table~\ref{table:comparison}).  In this paper, we analyze the Gaussian homotopy method for tensor PCA, and prove that it matches the best known signal-to-noise performance. \cite{hopkins2015tensor} also showed a lowerbound that no degree-$4$ (or lower) sum-of-squares algorithm can achieve better signal-to-noise ratio, implying that our analysis is likely to be tight. 
  
\begin{table*}
\begin{tabular}{|l| l| l| l|}\hline
Method & Bound on $\tau$ &Time   & Space   \\ \hline
{\bf Power method + initialization + noise injection (ours)} & $\tilde{\Omega}(n^{3/4}) $ & $\tilde{O}(n^3)$  & $\tilde{O}(n)$\\ \hline
Power method, random initialization& $\tilde{\Omega}(n)$ & $\tilde{O}(n^3)$  & $\tilde{O}(n)$\\ \hline
Sum-of-Squares  & $\tilde{\Omega}(n^{3/4})$ & $>\Omega(n^6)$ & $>\Omega(n^6)$ \\ \hline
Recover and Certify  & $\tilde{\Omega}(n^{3/4})$ & $\tilde{O}(n^5)$ & $O(n^4)$ \\ \hline
Eigendecomposition of flattened matrix   & $\tilde{\Omega}(n^{3/4})$ & $\tilde{O}(n^3)$ & $\tilde{O}(n^2)$\\ \hline
Information-theoretic & $\tilde{\Omega}(\sqrt{n})$ & Exp & $O(n)$\\ \hline
\end{tabular}\caption{Table of comparison of various methods for tensor PCA. Here space does not include the tensor itself. The power method with random initialization was analyzed in~\cite{richard2014statistical}. sum-of-squares, Recover and Certify, and flattened tensor were analyzed in~\cite{hopkins2015tensor}.
%\aacomment{Rong: please verify and add more. you know details about david's paper. } 
 }\label{table:comparison}
\end{table*} 

%\paragraph{Our Results} %In order to analyze homotopy method for tensor PCA, we make the following Independence Assumption
%
%\begin{assumption*}[informal] For every point $x$ that the homotopy method visits, the noise tensor $A$ is not adversarially correlated with $x$ and $v$.
%\end{assumption*}
%
%This in particular means that the trilinear forms $A(x,x,v) = \sum_{i,j,k} A_{i,j,k}x_ix_jv_k$ is not much larger than its standard deviation. We formalize this assumption in Section~\ref{sec:simplealg}. Intuitively, this assumption is reasonable, because $A$ is a purely random tensor, and these trilinear forms will be bounded by their standard deviations for most of the points $x$.

\subsection{Contributions}

We analyze a simple variant of the popular tensor power method, which is a local search method for finding the best rank-$1$ approximation of the input tensor. We modify it by introducing a specific initialization and injecting appropriate  random  noise in each iteration. This runs almost in linear time; see Table~\ref{table:comparison} for more details.

\begin{theorem}[informal] There is an almost linear time algorithm for tensor PCA that finds the signal $\bm v$ as long as the signal strength $\tau = \tilde{\Omega}(n^{3/4})$. 
\end{theorem}

Our algorithm achieves the {\em best possible trade-offs} among all known algorithms (see Table~\ref{table:comparison}). 

Our algorithm is inspired by the homotopy framework. 
In particular, we establish a phase transition along the homotopy path.
 
\begin{theorem}[informal] Under a plausible independence conjecture, there is a threshold $\theta$ such that if the radius of smoothing is significantly larger than $\theta$, the smoothed function will have a unique local and global maximum. If the radius of smoothing is smaller, then the smoothed function can have multiple local maxima, but one of them is close to the signal vector $\bm v$.
\end{theorem}

The above result allows us to skip the intermediate steps in the homotopy path. We only need two end points of the homotopy path: the original objective function with no smoothing and with an infinite amount of smoothing. The optimal solution for the latter can be obtained through any local search method; in fact, in our case, it has a closed form. This serves as initialization for the  original objective function.
%Our algorithm uses two regimes of homotopy: in the first regime the radius of smoothing is very large, and we use that to obtain a good initialization; in the second regime the radius of smoothing is small, and our algorithm use that to find the signal $\bm v$.
%The initialization for the power method is obtained as follows: the original objective function is smoothed to an infinite extent and its optimal solution can be obtained in closed form.  
In the proof we also design a new noise injection procedure that breaks the dependency between the steps. This allows for simpler analysis and our algorithm does not rely on the independence conjecture. We discuss this in more detail in Section~\ref{sec:resampling}.

%
%\WAS{To prove the above theorem we require an Independence Assumption, which roughly says the noise tensor $\bm A$ does not adversarially correlate with the signal $\bm v$ for all points on the homotopy path. We give a noise injection procedure under which this assumption holds for our algorithm. This is standard in analyzing nonconvex optimization algorithms. As an example, previous works on alternating minimization for matrix completion \citep{jain2013low} relied  on the availability of different subsamples in different iterations  to obtain the theoretical guarantees. Our noise injection procedure is very similar, however this is the first application of this idea for the case of Gaussian noise.} \IS{IS IT BETTER TO KEEP THIS PARAGRAPH HERE OR MOVE IT TO LATER SECTIONS?}
%\IS{Rong: How about we shorten this? In the proof we also design a new noise injection procedure that breaks the dependency between the steps and allows for simpler analysis. We discuss this in more detail in Section~\ref{sec:resampling}.}\aacomment{I agree}

The comparison of all the current algorithms for tensor PCA is given in Table~\ref{table:comparison}. Note that the space in the table does not include the space for storing the tensor, this is because the more practical algorithms only access the tensor for a very small number of passes, which allows the algorithms to be implemented online and do not need to keep the whole tensor in the memory. We see that our algorithm has the best performance across all the measures. In our synthetic experiments (see Section~\ref{sec:exp}, we find that our method significantly outperforms the other methods: it converges to a better solution faster and with a lower variance.  
 \vspace*{-0.1in}
%
%
% Under this assumption we can show
%
%\begin{theorem}[informal] There is a threshold $T$ where if the amount of smoothing is above $T$, the objective function for tensor PCA (Equation (\ref) in Section~) has a unique local optimum. If the amount of smoothing is below $T$, then there can be multiple local optima, and the homotopy algorithm will converge to the one near $v$.
%\end{theorem}
%
%Using this observation, we can design a very simple algorithm that achieves near optimal guarantees for tensor PCA
%
%\begin{theorem} When $t \ge n^{3/4}\log^? n$, assuming the Uncorrelated Assumption there is a simple homotopy algorithm that finds a vector close to $v$, and the algorithm only visits $O(\frac{\log n}{\log \log n})$ points.
%\end{theorem}
%
%Of course, in practice we cannot rely on the Uncorrelated Assumption. However, since the algorithm we propose only visits $O(\log n)$ points, we use a careful {\em resampling} argument that allows us to get rid of the Assumption.
%
%\begin{theorem} When $t \ge n^{3/4}\log^? n$, there is a simple homotopy algorithm for tensor PCA.
%\end{theorem}

\section{Preliminaries}
\label{sec:prelim}
In this section, we formally define the tensor PCA problem and its associated objective function. Then we show how to compute the smoothed versions of these objective functions. %Finally we discuss how these smoothed functions are used in the homotopy method, and introduce the Uncorrelated Assumption.

\subsection{Tensors and Polynomials}

Tensors are higher dimensional generalization of matrices. In this paper we focus on 3rd order tensors, which correspond to a 3 dimensional arrays. Given a vector $\bm v\in \R^n$, similar to rank one matrices $\bm v \bm v^\top$, we consider rank 1 tensors $\bm v^{\otimes 3}$ to be a $n\times n\times n$ array whose $i,j,k$-th entry is equal to $\bm v_i \bm v_j \bm v_k$.

For a matrix $\bm M$, we often consider the quadratic form it defines: $\bm x^\top \bm M \bm x$. Similarly, for a tensor $\bm T \in \R^{n\times n\times n}$, we define a degree 3 polynomial $\bm T(\bm x, \bm x, \bm x) = \sum_{i,j,k=1}^n \bm T_{i,j,k} \bm x_i \bm x_j \bm x_k$. This polynomial is just a special trilinear form defined by the tensor. Given three vectors $\bm x,\bm y, \bm z$, the trilinear form $\bm T(\bm x, \bm y, \bm z) = \sum_{i,j,k=1}^n \bm T_{i,j,k} \bm x_i \bm y_j \bm z_k$. Using this trilinear form, we can also consider the tensor as an operator that maps vectors to matrices, or two vectors into a single vector. In particular, $\bm T(\bm x, :, :)$ is a matrix whose $i,j$-th entry is equal to $\bm T(\bm x, \bm e_i, \bm e_j)$ where $\bm e_i$ is the $i$-th basis vector. Similarly, $\bm T(\bm x, \bm y, :)$ is a vector whose $i$-th coordinate is equal to $\bm T(\bm x, \bm y, \bm e_i)$.

Since the tensor $\bm T$ we consider is not symmetric ($\bm A_{ijk}$ is not necessarily equal to $\bm A_{jik}$ or other permutations), we also define the symmetric operator 

$$\delta(\bm x) = \bm A(\bm x,\bm x,:)+\bm A(\bm x,:,\bm x)+\bm A(:,\bm x,\bm x).$$

\subsection{Objective Functions for Tensor PCA}

We first define the tensor PCA problem formally.

    \begin{definition}[Tensor PCA]
        Given input tensor $\bm T = \tau \cdot \bm v^{\otimes 3} + \bm A$, where $\bm v \in \mathbb R^n$ is an arbitrary unit vector, $\tau \geq 0$ is the signal-to-noise ratio, and $\bm A$ is a random noise tensor with iid standard Gaussian entries, recover the signal $\bm v$ approximately (find a vector $\|\bm x\| = 1$ such that $\inner{\bm x, \bm v} \ge 0.8$).
    \end{definition}

Similar to the Matrix PCA where we maximize the quadratic form, for tensor PCA we can focus on optimizing the degree 3 polynomial $f(\bm x) = \bm T(\bm x, \bm x, \bm x)$ over the unit sphere.
\begin{align}
\max_{\bm x} \quad &f(\bm x) = \tau \langle \bm v, \bm x \rangle^3 + \bm A(\bm x,\bm x,\bm x) \label{eq:obj}\\
& \|\bm x\| = 1 \nonumber 
\end{align}

The optimal value of this program is known as the spectral norm of the tensor. It is often solved in practice by tensor power method. \cite{richard2014statistical} noticed that:

\begin{theorem}
When $\tau \ge C\sqrt{n}$ for large constant $C$, the global optimum of (\ref{eq:obj}) is close to the signal $\bm v$.
\end{theorem} 

Unfortunately, solving this optimization problem is NP-hard in the worst-case~\citep{Hillar}. Currently, the best known algorithm uses sum-of-squares hierarchy and works when $\tau \ge Cn^{3/4}$. There is a huge gap between what's achievable information theoretically ($O(\sqrt{n})$) and what can be achieved algorithmically ($\Omega(n^{3/4})$).

\subsection{Gaussian Smoothing for the Objective Function}

%    Note that the input tensor can be viewed as a function $f : \mathcal X \to \mathbb R$, where $\mathcal X = \mathbb R^n$,
%    \[
%        f(\bm x) = \bm T(\bm x,\bm x,\bm x) = \tau \langle \bm v, \bm x \rangle^3 + \bm A(\bm x,\bm x,\bm x)
%    \]
%    \Ynote{In order to recover $\bm v$, it is equivalent to find the global maximum of $f(\bm x)$ constrained on $\|\bm x\|_2 = 1$???}

Guaranteed homotopy methods rely on smoothing the objective function by the Gaussian kernel ~\citep{MobahiLink, HosseinAAAI}. More precisely, smoothing the objective (\ref{eq:obj}) requires convolving it with the Gaussian kernel.
%    We consider an embedding of $f$ into a family of Gaussian smoothed functions 
Let $g : \mathcal X \times \mathcal \R^+ \to \mathbb R$ be a mapping such that
    \[
        g(\bm x, t) = [f \star k_{t}](\bm x)
    \]
    Here, $k_{t}$ is the Gaussian density function for $\mathcal{N}(\boldsymbol{0}, t^2 \boldsymbol{I}_n)$, satisfying 
    \[
        k_{t}(\bm x) = \dfrac{1}{(\sqrt{2 \pi}t)^n} \cdot e^{-\frac{\|\bm x\|_2^2}{2 t^2}}.
    \]
    
It is known that convolution of polynomials with the Gaussian kernel has a closed form expression \citep{ClosedGauss16}. In particular, the objective function of the Tensor PCA has the following smoothed form.

\begin{lemma} [Smoothed Tensor PCA Objective]
\label{lem:smooth} The smoothed objective has the form
$$ g(\bm x, t) = \tau \langle \bm v, \bm x \rangle^3 + t^2 \langle 3 \tau \bm v + \bm u, \bm x \rangle + \bm A(\bm x,\bm x,\bm x),$$
where the vector $\bm u$ is defined by $\bm u_j = \sum_{i=1}^n (\bm A_{iij}+\bm A_{iji} + \bm A_{jii})$.
Moreover, it is easy to compute vector $\bm z = 3\tau \bm v+\bm u$ given just the tensor $\bm T$, as $\forall j, \bm z_j=\sum_{i=1}^n (\bm T_{iij}+\bm T_{iji}+\bm T_{jii})$.
\end{lemma}

The proof of this Lemma is based on interpreting the convolution as an expectation $\E_{y\sim N(\bm 0, \mathbf I_n)}[f(x+y)]$. We defer the detailed calculation to Appendix~\ref{app:lemma1}

\section{Tensor PCA by Homotopy Initialization}
\label{sec:simplealg}
In this section we give a simple smart initialization algorithm for tensor PCA. Our algorithm only uses two points in homotopy path \--- the infinite smoothing $t\rightarrow \infty$ and the no smoothing $t \rightarrow 0$. This is inspired by our full analysis of the homotopy path (see Section~\ref{sec:path}), where we show there is a {\em phase transition} in the homotopy path. When the smoothing parameter is larger than a threshold, the function behaves like the infinite smoothing case; when the smoothing parameter is smaller than the threshold, the function behaves like the no smoothing case.

%    Notice that in the analysis of the homotopy path, we observe that when $t^2$ is large, since the local minimizer still exists and does not change a lot, the homotopy solution is close to the initial solution. However, once $t^2$ decreases to $\Theta(1/n)$, at some point, the local minimizer turns into a saddle point and then disappear, with high probability, a gradient descent algorithm leads the solution to a local minimizer with high correlation with the target vector $\bm v$. Finally, such minimizer converges to a vector closed to $\bm v$ as $t$ goes to $0$.
    
%    This result provides us an intuition to design a simple algorithm without a full implementation of the homotopy process. In this section, we show that with a careful chosen $t$ to smooth the tensor, we can apply a gradient descent algorithm to the smoothed function directly to recover the signal $\bm v$.
    
Recall that the smoothed function $g(\bm x, t)$ is:
    \begin{equation}
         g(\bm x, t) = \tau \langle \bm v, \bm x \rangle^3 + t^2 \langle 3 \tau \bm v + \bm u, \bm x \rangle + \bm A(\bm x,\bm x,\bm x)  
    \end{equation}
    with $\bm u$ as a vector such that $\bm u_j = \sum_i \bm A_{iij} + \bm A_{iji} + \bm A_{jii}$. When $t \rightarrow \infty$, the solution of the smoothed problem has the special form $\bm x^{\dag} = \frac{3 \tau \bm v + \bm u}{\|3 \tau \bm v + \bm u\|}$. That is because the term $t^2 \langle 3 \tau \bm v + \bm u, \bm x \rangle$ dominates $g$ and thus its maximizer under $\|\bm x \|_2 = 1$ yields $\bm x^{\dag}$. 

Note that by Lemma~\ref{lem:smooth}, we can compute vector $\bm z$ $\bm z_j = \sum_{i=1}^n \bm T_{iij}+\bm T_{iji} + \bm T_{jii}$, and we know $\bm z = 3\tau \bm v+\bm u$. Therefore we know $\bm x^\dag = \frac{\bm z}{\|\bm z\|_2}$ can be computed just from the tensor. We use this point as an initialization, and then run power method on the original function. The resulting algorithm is described in Algorithm~\ref{algo:simple}. 
    
    In order to analyze the algorithm, we use the following {\em independence} condition, which states that the ``random''-looking vectors $\bm u$ and $\delta(\bm x^p) = \bm A(\bm x^p, \bm x^p, :)+\bm A(\bm x^p, :, \bm x^p)+\bm A(:, \bm x^p, \bm x^p)$ indeed have some properties satisfied by random vectors:
\begin{condition}\label{assump}[Independence Condition]
     The norm and correlation with $\bm v$ for the vectors $\bm u$ and $\delta(\bm x^p)$ are not far from expectation. More precisely: (1) $\|\bm u\|_2 = O(n \sqrt{m})$ and $|\langle \bm u, \bm v \rangle| = O(\sqrt {n m \log n})$; (2) for the sequence computed by Algorithm \ref{algo:simple}, $\bm x^0, \bm x^1, \cdots, \bm x^m$, $\forall 0 \leq p \leq m$, $\|\delta(\bm x^p)\|_2 = O(\sqrt{n m}) \|\bm x^p\|_2^2$ and $|\langle \delta(\bm x^p), \bm v \rangle| = O(\sqrt{m \log n}) \|\bm x^p\|_2^2$.
\end{condition}

Note that if in every step of the algorithm, the noise tensor $\bm A$ is {\em resampled} to be a fresh random tensor, independent of the previous step $\bm x^p$, then $\delta(\bm x^p)$ is just a random Gaussian vector. In this case the condition is trivially satisfied. Of course, in reality $\bm x^i$'s are dependent on $\bm A$. However, we are able to modify the algorithm by a {\em noise injection} procedure, that adds more noise to the tensor $\bm T$, and make the noise tensor ``look'' as if they were independent. The extra dependency on $m$ in Condition~\ref{assump} comes from noise injection procedure and will be more clear in Section~\ref{sec:resampling}.  We will first show the correctness of the algorithm assuming independence condition here, and in Section~\ref{sec:resampling} we discuss the noise injection procedure and prove the independence condition.

    \begin{theorem}%        For any $t = o(n^{-\frac 1 2})$, we can apply a gradient descent algorithm on $g(\bm x, t)$ with $\bm x^{\infty}$ as the initial solution and infinite-step update. Such algorithm converges to a vector close to $\bm v$ in $O(\frac{\log n}{c \log \log n})$ iterations.
    \label{thm:converge}
When $\tau \ge Cn^{3/4}\log n$ for a large enough constant $C$, under the Independence Condition (Condition~\ref{assump}), Algorithm~\ref{algo:simple} finds a vector $\bm x^m$ such that $\inner{\bm x^m, \bm v} \geq 0.8$ in $O(\log \log n)$ iterations.    
    \end{theorem}

    \begin{algorithm2e}[h] 
        \KwIn{Tensor $\bm T = \tau \cdot \bm v^{\otimes 3} + \bm A$;}
        \KwOut{Approximation of $\bm v$;}
        $m = O(\log \log n)$;\\

        $\forall~j, \bm x^0_j = \sum_i \bm T_{iij} + \bm T_{iji} + \bm T_{jii}$;\\
        $\bm x^0 = \bm x^0 / \|\bm x^0\|$; \hfill// Now $\bm x^0 = \bm x^\dag$\\
        \For {$k = 0$ to $m$} {
            $\bm x^{k+1} = \bm T(\bm x^{k},\bm x^{k},:) + \bm T(\bm x^k,:,\bm x^k) + \bm T(:,\bm x^k,\bm x^k)$; \\
            $\bm x^{k+1} = \bm x^{k+1} / \|\bm x^{k+1}\|$;
        }
        {\Return $\bm x^m$};
        \caption{Tensor PCA by Homotopy Initialization}
        \label{algo:simple}
    \end{algorithm2e}   

The main idea is to show the correlation of $\bm x^k$ and $\bm v$ increases in every step. In order to do this, first notice that the initial point $\bm x^\dag$ itself is equal to a normalization of $3\tau \bm v+\bm u$, where the norm of $\bm u$ and its correlation with $\bm v$ are all bounded by the Independence Condition. It is easy to check that $\inner{\bm x^0,\bm v} \gg n^{-1/4}$, which is already non-trivial because a random vector would only have correlation around $n^{-1/2}$. For the later iterations, let $\hat{\bm x}^k$ be the vector $\bm x^k$ before normalization and we have $\hat{\bm x}^{k+1} = 3 \tau \langle \bm v, \bm x^k \rangle^2 \bm v + \delta(\bm x^k)$. Notice that the first term is in the direction $\bm v$, and the Independence Condition bounds the norm and correlation with $\bm v$ for the second term. We can show that the correlation with $\bm v$ increases in every iteration, because the initial point already has a large inner product with $\bm v$. The detailed proof is deferred to Appendix~\ref{app:theorem3}.

\subsection{Noise Injection Procedure}
\label{sec:resampling}
\begin{algorithm2e}[h] 
        \KwIn{Tensor $\bm T = \tau \cdot \bm v^{\otimes 3} + \bm A$;}
        \KwOut{Approximation of $\bm v$;}
        $m = O(\log \log n)$;\\
        Sample $\bm B^0, \bm B^1, ..., \bm B^{m - 1} \in \R^{n\times n\times n}$ whose entries are $\mathcal N(0,m)$.\\
        Let $\overline{\bm B} = \frac{1}{m} \sum_{p=0}^{m-1} \bm B^p$.\\
        Let $\bm T^p = \bm T-\bar{\bm B}+\bm B^p$ \\
        $\forall~j, \bm x^0_j = \sum_i \bm T^0_{iij} + \bm T^0_{iji} + \bm T^0_{jii}$;\\
        $\bm x^0 = \bm x^0 / \|\bm x^0\|$;\\
        \For {$k = 0$ to $m-2$} {
            $\bm x^{k+1} = \bm T^{k+1}(\bm x^{k},\bm x^{k},:) + \bm T^{k+1}(\bm x^k,:,\bm x^k) + \bm T^{k+1}(:,\bm x^k,\bm x^k)$; \\
            $\bm x^{k+1} = \bm x^{k+1} / \|\bm x^{k+1}\|$;
        }
        {\Return $\bm x^{m-1}$};
        \caption{Tensor PCA with Homotopy Initialization and Noise Injection}
        \label{algo:resample}
    \end{algorithm2e}   

In order to prove the Independence Condition, we slightly modify the algorithm (see Algorithm~\ref{algo:resample}). In particular, we add more noise in every step as follows %Suppose the algorithm runs in $m$ steps and we change the above algorithm as follows
    \begin{itemize}
        \item Get the input tensor $\bm T = \tau \cdot \bm v^{\otimes 3} + \bm A$;
        \item Draw a sequence of $\bm B^p \in \mathbb R^{n \otimes 3}$ such that $\bm B_{ijk}^p \sim \mathcal N(0, m)$;
        \item Let $\bm T^p =  \bm T - \overline{\bm B} + \bm B^p$ with $\overline{\bm B} = \frac{1}{m} \sum_{p=0}^{m-1} \bm B^p$, run Algorithm \ref{algo:resample} by using $\bm T^p$ in the $p$-th iteration;
    \end{itemize}
    Intuitively, by adding more noise the new noise will overwhelm the original noise $\bm A$, and every time it looks like a fresh random noise. We prove this formally by the following lemma:

    \begin{lemma}
    \label{lem:resample}
        Let the sequence $\bm T^0, \cdots, \bm T^{m-1}$ be generated according to Algorithm~\ref{algo:resample}. Let $\bm Q^i = \tau \bm v^{\otimes 3}+\bm C^i$, where $\bm C^i$'s are tensors with independent Gaussian entries. Each entry in $\bm C^i$ is distributed as $N(0,m)$. The two sets of variables $\{\bm T^i\}$ and $\{\bm Q^i\}$ has the same distribution. 
    \end{lemma}
    
    This Lemma states that after our noise injection procedure, the tensors $\bm T^0,...,\bm T^{m-1}$ look {\em exactly the same} as tensors where the noise $\bm A$ is sampled independently.
The basic idea for this lemma is that for two multivariate Gaussians to have the same distribution, we only need to show that they have the same first and second moments. We defer the details to Appendix~\ref{app:theorem3}.
   
    Using Lemma~\ref{lem:resample} we can create a sequence of $T^p$ such that its noise tensor $\bm A^p = \bm A - \overline{\bm B} + \bm B^p$ is redrawn independently and each element is according to $\mathcal N(0, m)$. Now, because each $\bm T^i$ behave as if it is drawn independently, we can prove the Independence Condition:
    
\begin{lemma}[Noise Injection]\label{lem:resamplingassump} Let $\bm T^p$ be generated according to Algorithm~\ref{algo:resample} and $\bm A^p = \bm T^p - \tau \bm v^{\otimes 3}$. Let $\bm u^0$ be a vector such that $\bm u^0_j = \sum_i \bm A_{iij}^0 + \bm A_{iji}^0 + \bm A_{jii}^0$, and $\delta^p(\bm x^p) = \bm A^p(\bm x^p, \bm x^p, :)+\bm A^p(\bm x^p, :, \bm x^p)+\bm A^p(:, \bm x^p, \bm x^p)$. With high probability\footnote{Throughout this paper by ``with high probability'' we mean the probability is at least $1-1/n^C$ for a large constant $C$.}, (1) $\|\bm u^0\|_2 = \Theta(n \sqrt{m})$ and $|\langle \bm u^0, \bm v \rangle| = O(\sqrt {n m \log n})$; (2) for the sequence computed by Algorithm \ref{algo:resample}, $\bm x^0, \bm x^1, \cdots, \bm x^{m-1}$, $\forall~ 0 \leq p \leq m-1$, $\|\delta^p(\bm x^p)\|_2 = \Theta(\sqrt{n m}) \|\bm x^p\|_2^2$ and $|\langle \delta^p(\bm x^p), \bm v \rangle| = O(\sqrt{m \log n}) \|\bm x^p\|_2^2$. As a result Condition~\ref{assump} is satisfied.
\end{lemma}

This Lemma is now true because by Lemma~\ref{lem:resample}, we know the noise tensors $\bm A^p$ is independent of $\bm A^0,...,\bm A^{p-1}$. As a result $\bm A^p$ is {\em independent} of $\bm x^p$! This lemma then follows immediately from standard concentration inequalities. We defer the full proof to Appendix~\ref{app:theorem3}.

The noise injection technique is mostly a technicality that we need in order to make sure different steps are independent. This is standard in analyzing nonconvex optimization algorithms. As an example, previous works on alternating minimization for matrix completion \citep{jain2013low} relied on the availability of different subsamples in different iterations to obtain the theoretical guarantees. Our noise injection procedure is very similar, however this is the first application of this idea for the case of Gaussian noise. The main usage of the noise injection is to get rid of the dependence of the noise matrix between different iterations. Moreover, this technique is designed to simplify the proof and rarely used in the real applications. In practice, an algorithm without noise injection, like Algorithm \ref{algo:simple}, usually performs well enough.

Combining Lemma~\ref{lem:resamplingassump} and Theorem~\ref{thm:converge}, we know Algorithm~\ref{algo:resample} solves the tensor PCA problem when $\tau \ge Cn^{3/4}\log n$.

\begin{remark}[Estimation of the variance in practice]
In the above analysis, we assume the variance of entries of $\bm A$ is $1$. In practice, we can estimate the variance $\sigma^2$ of entries of $\bm A$ from $\bm T$ by computing its Frobenius norm. Note that when $\tau$ is large, the simple power method already performs well. The interesting case is when $\tau$ is small, say $\tau < n$. In this case, the square of the Frobenius norm of $\tau \bm v^{\otimes 3} = \tau^2$ while the square of the Frobenius norm of the noise matrix $\bm A$ in expectation is $\sigma^2 n^3$ with variance $\sigma^2 O(n^3)$. Therefore, we can get a good estimation of $\sigma^2$ by computing the square of the Frobenius norm of $\bm A$ divided by $n^3$.
\end{remark}

\section{Characterizing the Homotopy Path}

    \begin{figure*}[htbp]
        \begin{minipage}[b]{0.33\linewidth}
            \centering
            \includegraphics[width=0.9\linewidth]{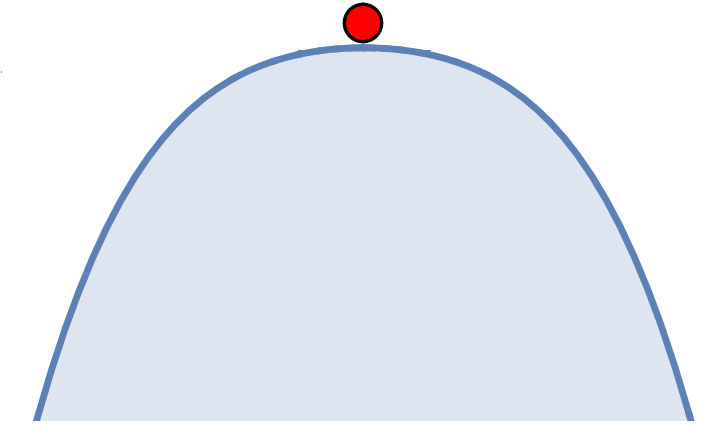}\\
            {(a) $t \gg n^{-1}$}
        \end{minipage}
        \begin{minipage}[b]{0.33\linewidth}
            \centering
            \includegraphics[width=0.9\linewidth]{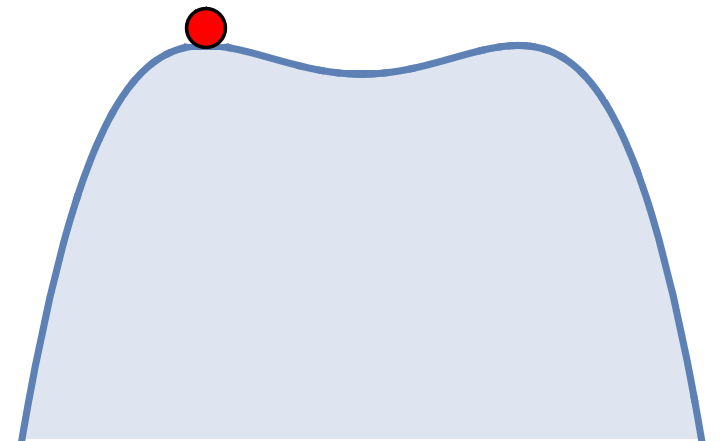}\\
            {(b) $t \approx n^{-1}$}
        \end{minipage}
        \begin{minipage}[b]{0.33\linewidth}
            \centering
            \includegraphics[width=0.9\linewidth]{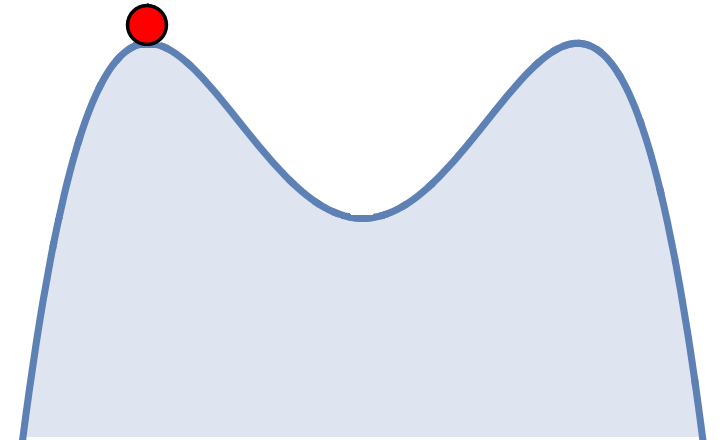}\\
            {(c) $t \ll n^{-1}$}
        \end{minipage}
        \caption{Phase Transition for a 1-d function}
        \label{fig:phasetransition}
    \end{figure*}

This section analyzes the behavior of the smoothed objective function $g$ as $t$ varies. Under a plausible conjecture, we prove that a phase transition occurs: when $t$ is large $g(\bm x ,t)$ behaves very similarly to $g(\bm x ,\infty)$ and when $t$ is small $g(\bm x ,t)$ behaves very similarly to $g(\bm x ,0)$.
This motivates the algorithms in the previous section, as the phase transition suggests the most important regimes are very large $t$ and $t=0$.

In this section we first describe how the homotopy method works in more details. Then we present an alternative objective function of Tensor PCA and derive its smoothed version. Finally, we prove that when $t \gg n^{-1}$, the smoothed function retains its maximizer around $\bm x^\dag$. However, when $t \ll n^{-1}$, the configuration of critical points change, with only one of the critical points being close to the solution $\bm v$. Importantly, we can find our way from the vicinity of $\bm x^\dag$ toward this critical point via the dominant curvature direction of the function.

\label{sec:path}
\subsection{Homotopy}

In the homotopy method, we start from the maximizer of the function with large amount of smoothing $t \rightarrow \infty$. We earlier denoted this maximizer as $\bm x^\dag$. Then we continuously decrease the amount of smoothing $t$, while following the maximizer throughout this process, until reaching $t = 0$. We call the path taken by the maximizer the homotopy path. It is formally defined as follows.

    \begin{definition}[Homotopy Path]
        A homotopy path $\bm x(t)$ is a continuous function $\bm x : \mathcal T \to \mathcal X$ satisfying $\lim_{t \to \infty} \bm x(t) = \bm x^{\dag}$ and $\forall ~ t \geq 0$, $\nabla g(\bm x(t), t) = 0$, where the gradient $\nabla$ is w.r.t. to the first argument of $g$.
    \end{definition}
    
    In practice, to search a homotopy path, one computes the initial point $\bm x^{\dag}$ by analytical derivation or numerical approximation as $\arg \max_{\bm x} g(\bm x, t)$ and then successively minimizes the smoothed functions over a finite sequence of decreasing numbers $t_0$ to $t_m$, where $t_0$ is sufficiently large, and $t_m=0$. The resulted procedure is listed in Algorithm \ref{algo:homotopy}.
   
    \begin{algorithm2e}[h] 
        \KwIn{$f: \mathcal X \to \mathbb R$, a sequence $t_0 > t_1 > \cdots > t_m = 0$.}
        \KwOut{A (good) local maximizer of $f$.}
        $\bm x^{t_0} = \mbox{global maximizer of }g(\bm x, t_0)$; \\
        \For {$k = 1$ to $m$} {
            $\bm x^{t_k} =$ Local maximizer of $g(\bm x; t_k)$, initialized at $\bm x^{t_{k - 1}}$.
        }
        \Return $\bm x^{t_m}$.
        \caption{Homotopy Method}
        \label{algo:homotopy}
    \end{algorithm2e}   
    
   % As discussed in ??, for a wide range of applications, the final point of the homotopy path often provides a good local minimizer of $f(\bm x)$.

\subsection{Alternative Objective Function and Its Smoothing}

   Turning a constrained problem into an unconstrained problem can facilitate the computation of the effective gradient and Hessian of $g(\bm x, t)$. In this section, we consider the alternative objective function: we modify $f(\bm x)$ by adding the penalty term $-\frac {3 \tau} 4 \|\bm x\|_2^4$:
    \[
        f_r(\bm x) = \tau \langle \bm v, \bm x \rangle^3 + \bm A(\bm x,\bm x,\bm x) - \frac {3 \tau} 4 \|\bm x\|_2^4
    \]

Thus we consider the following unconstrained optimization problem,
\begin{equation}
\max \quad f_r(\bm x) = \bm T(\bm x, \bm x, \bm x) - \frac{3\tau}{4}\|\bm x\|_2^4. \label{eq:obj2}
\end{equation}

If we fix the magnitude $\|\bm x\| = 1$, the function $f_r(\lambda \bm x)$ is $\lambda^3 \bm T(\bm x, \bm x, \bm x) - \frac{3\tau}{4} \lambda^4$. The optimizer of this is an increasing function of $\bm T(\bm x, \bm x, \bm x)$. Therefore the maximizer of (\ref{eq:obj2}) is exactly in the same direction as the constrained problem (\ref{eq:obj}). The $3\tau/4$ factor here is just to make sure the optimal solution has roughly unit norm; in practice we can choose any coefficient in front of $\|\bm x\|^4$ and the solution will only differ by scaling. 

%The vector $\bm x$ can be always written as the product of its magnitude and direction $\bm x = \lambda \bm y$, where $\|\bm y\|=1$. Thus, the function $f_r(\lambda \bm x)$ becomes $\lambda^3 \bm T(\bm y, \bm y, \bm y) - \lambda^4$. The optimizer of this is an increasing function of $\bm T(\bm y, \bm y, \bm y)$. Therefore the maximizer of (\ref{eq:obj2}) is exactly in the same direction as the constrained problem (\ref{eq:obj}). The $3\tau/4$ factor here is just to make sure the optimal solution have roughly unit norm, in practice we can choose any coefficient in front of $\|\bm x\|^4$ and the solution will only differ by scaling. 

    Moreover, note that if in the absence of noise tensor $\bm A$, then
    \[
        \nabla f_r(\bm x) = 3 \tau \langle \bm v, \bm x \rangle^2 \bm v - \frac {3 \tau} 4 \cdot 4 \|\bm x\|_2^2 \bm x \\
    \]
    To get the stationary point, we have
    \[
        \bm x = \frac {3 \tau} 4 \cdot \frac {\langle \bm v, \bm x \rangle^2}{\frac {3 \tau} 4 \cdot \|\bm x\|_2^2} \bm v = \bm v 
    \]
    Therefore, the new function $f_r(\bm x)$ is defined on $\mathbb R^n$ and the maximizer of $\mathbb R^n$ is close to $\bm v$. We also compute the smoothed version of this problem:
   
    \begin{lemma}[Smoothed Alternative Objective]\label{lem:smoothalternative}
    The smoothed version of the alternative objective is
    \begin{align*}
        g_r(\bm x, t) = \tau \langle \bm v, \bm x \rangle^3 + t^2 \langle 3 \tau \bm v + \bm u, \bm x \rangle + \bm A(\bm x,\bm x,\bm x) - \frac {3 \tau} 4 \left(\|x\|_2^4 + 2 t^2 (n+2) \|x\|_2^2 + t^4 (n^2+2n)\right)
    \end{align*}
    Its gradient and Hessian are equal to
\begin{equation}\label{eq:gradient}
        \nabla g_r(\bm x, t) = 3 \tau \langle \bm v, \bm x \rangle^2 \bm v + t^2 (3 \tau \bm v + \bm u) 
         + \delta(\bm x) - 3 \tau (\|\bm x\|_2^2 \bm x + t^2 (n+2) \bm x).
    \end{equation}
    and
    \begin{align}
        \nabla^2 g_r(\bm x, t) &~=  -3 \tau ((\|\bm x\|_2^2 + t^2 (n+2)) \bm I - 2 \langle \bm v, \bm x \rangle \bm v \bm v^T + 2 \bm x \bm x^T) \nonumber \\
        &~+ P_{sym}[\bm A(\bm x, :, :) + \bm A(:, \bm x, :) + \bm A(:, :, \bm x)]. \label{eq:hessian}
    \end{align}  
    Here $P_{sym} \bm M = \frac{\bm M+\bm M^\top}{2}$ is the projection to symmetric matrices.
    \end{lemma}

The proof of this Lemma is very similar to Lemma~\ref{lem:smooth} and is deferred to Appendix~\ref{app:path}.

\subsection{Phase Transition on the Homotopy Path}

Notice that when $t\to \infty$, the dominating terms in $g_r(\bm x,t)$ are $t^2$ terms (the only $t^4$ term is a constant). Therefore, $g_r(\bm x,t)$ forms a quadratic function, so it has a unique global maximizer equal to $\frac{3\tau \bm v+\bm u}{3 \tau (n+2)}$, denoted as $\bm x^\dag$. Notice that this vector has different norm compared to the $\bm x^\dag$ in previous section.

Before we state the Theorem, we need a counterpart of the Independence Condition. We call this the Strong Independence Conjecture:

   \begin{conjecture}\label{assump2}[Strong Independence Conjecture]
        Suppose $\bm T = \tau \bm v^{\otimes 3} + \bm A$ and $\bm u_j = \bm A_{iij}+\bm A_{iji}+\bm A_{jii}$, $\delta(\bm x) = \bm A(\bm x,\bm x,:)+\bm A(\bm x,:,\bm x)+\bm A(:, \bm x,\bm x)$ be defined as before. With high probability, (1) $\|\bm u\|_2 = \Theta(n)$ and $|\langle \bm u, \bm v \rangle| = O(\sqrt {n\log n})$; (2) for all $\bm x^{t_k}$ on the homotopy path, $\|\delta(\bm x^{t_k})\|_2 = \Theta(\sqrt{n}) \|\bm x^{t_k}\|_2^2$ and $|\langle \delta(\bm x^{t_k}), \bm v \rangle| = O(\sqrt{\log n}) \|\bm x^{t_k}\|_2^2$, 
\end{conjecture}

Intuitively, this assumes that the noise is not adversarially correlated with the signal $\bm v$ on the entire homotopy path. The main difference between the strong independence conjecture and the weak independence conjecture is that they apply to different algorithms with different number of iterations. The strong independence conjecture applies to the general Homotopy method, which may have a large number of iterations, and thus a conjecture that depends on the number of iterations does not provide us any useful properties. We use the strong independence conjecture to analyze the general Homotopy method to gain intuitions in order to design our algorithm. The weak conjecture is for our Algorithm \ref{algo:simple}, which only has $O(\log \log n)$ rounds, and can be satisfied using the noise injection technique. Although we cannot use noise injection to prove the strong independence conjecture, similar conjectures are often used to get intuitions about optimization problems~\citep{donoho2009message,javanmard2013state,choromanska2015loss}.

\begin{theorem}\label{thm:homotopypath}
Assuming the Strong Independence Conjecture (Conjecture~\ref{assump2}), when $\tau = n^{3/4}\log n$,
\begin{enumerate}
\item When $t \ge Cn^{-1}$ for a large enough constant $C$, there exists a local maximizer $\bm x^t$ of $g_r(\bm x, t)$ such that $\|\bm x^t - \bm x^\dag\|_2 = o(1) \|\bm x^\dag\|_2$;
\item When $t < n^{-1}\log^{-2} n$, we know there are two types of local maximizers $\bm x^t$:
\begin{itemize}
\item $\|\bm x^t\|_2 = \Theta(1)$ and $\inner{\bm v, \bm x^t} = \Theta(1)$. This corresponds to a local maximizer near the true signal $\bm v$.
\item $\|\bm x^t\|_2 = \Theta(n^{-\frac 1 4} \log^{-1} n)$ and $\inner{\bm v, \bm x^t} = O(n^{-\frac 1 2} \log^{-1} n)$. These local maximizers have poor correlation with the true signal.
\end{itemize}
\item When $t < n^{-1}\log^{-2} n$, let $\bm b$ be the top eigenvector of $\nabla^2 (g_r(\bm x^\dag, t))$, we know $\sin\theta(\bm b, \bm v) \le 1/\log^2 n$.
\end{enumerate}
\end{theorem}

Intuitively, this theorem shows that in the process of homotopy method, if we consider a continuous path in the sense that $t_{k+1} - t_k$ is close to $0$ for all $k$, then (1) at the beginning, $\bm x^k$ is close to $\bm x^\dag$; (2) at some point $k^*$, $\bm x^{k^*}$ is a saddle point in the function $g(\bm x, t_{k^*+1})$ and from the saddle point we are very likely to follow the Hessian direction to actually converge to the good local maximizer near the signal. This phenomenon is illustrated in Figure~\ref{fig:phasetransition}:

Figure~\ref{fig:phasetransition}(a) has large smoothing parameter, and the function has a unique local/global maximizer. Figure~\ref{fig:phasetransition}(b) has medium smoothing parameter, the original global maximizer now behaves like a local minimizer in one dimension, but it in general could be a saddle point in high dimensions. The Hessian at this point leads the direction of the homotopy path. In Figure~\ref{fig:phasetransition}(c) the smoothing is small and the algorithm should go to a different maximizer.

\section{Experiments}
\label{sec:exp}
    \begin{figure*}[!htbp]
        \begin{minipage}[b]{0.33\linewidth}
            \centering
            \includegraphics[width=0.9\linewidth]{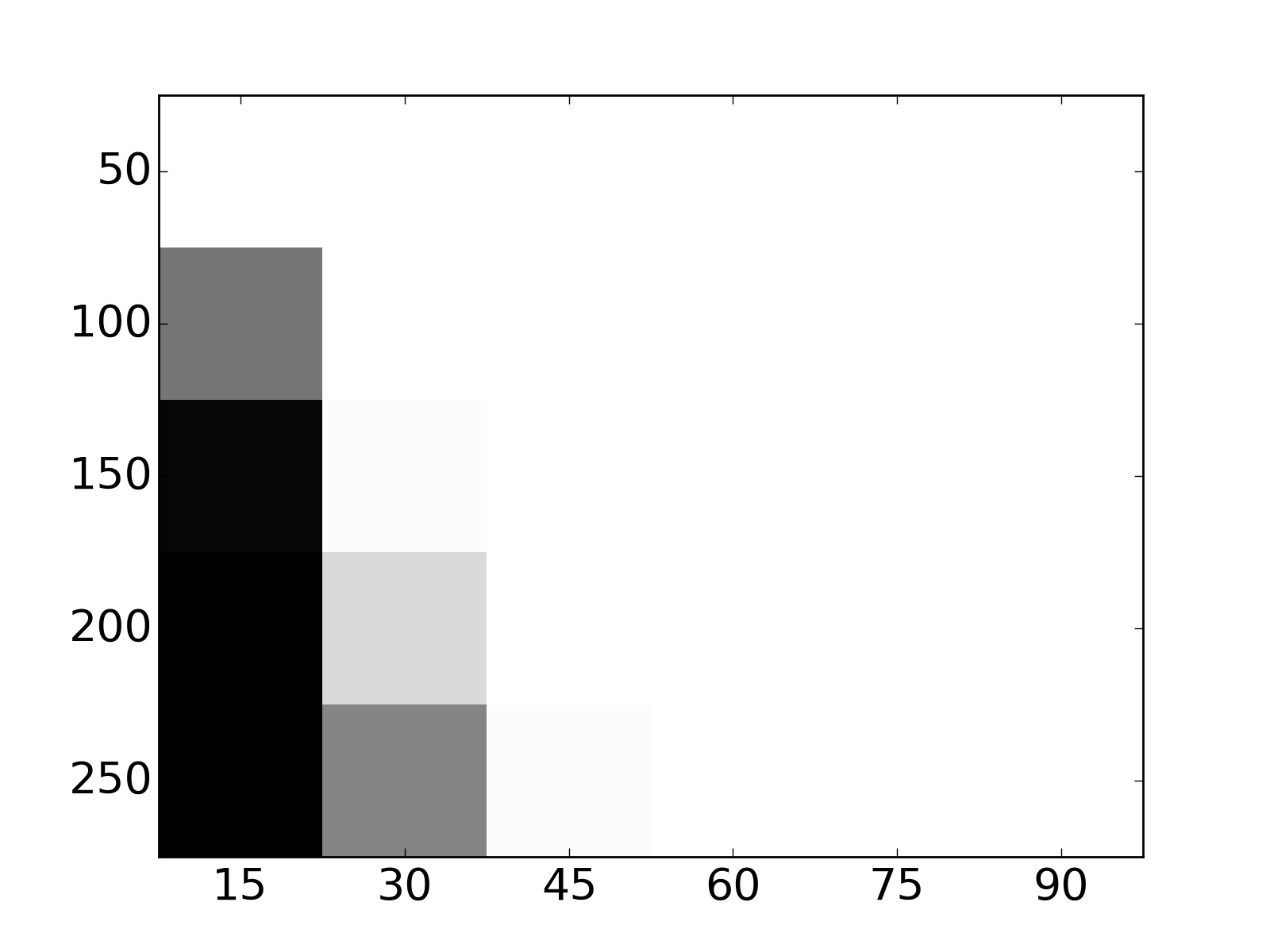}
            {Homotopy PCA}
            \label{fig:homotopy}
        \end{minipage}
        \begin{minipage}[b]{0.33\linewidth}
            \centering
            \includegraphics[width=0.9\linewidth]{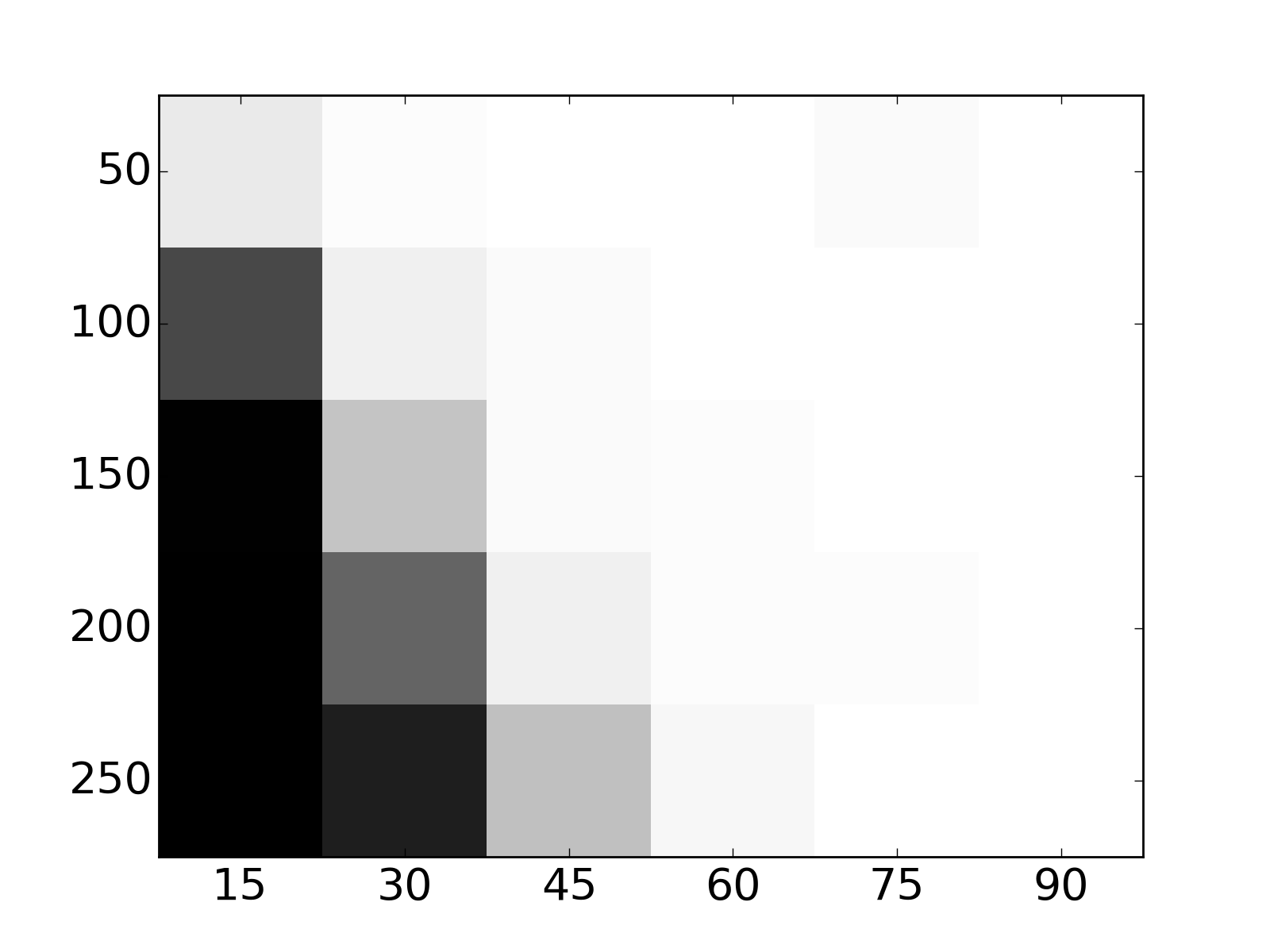}
            {Power Method}
            \label{fig:power}
        \end{minipage}
        \begin{minipage}[b]{0.33\linewidth}
            \centering
            \includegraphics[width=0.9\linewidth]{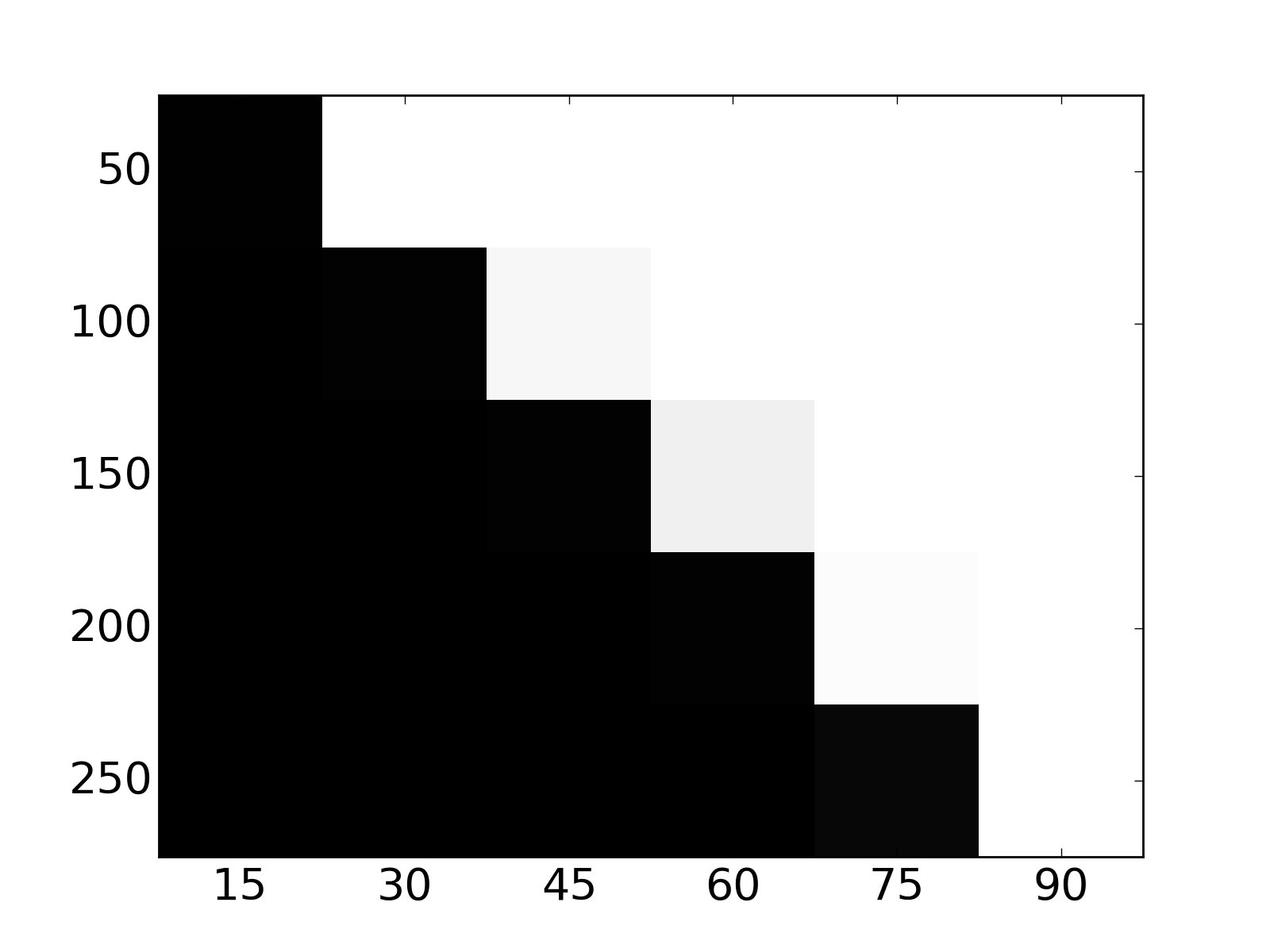}
            {Flatten Algorithm}
            \label{fig:flatten}
        \end{minipage}
        \caption{Success probabilities for the algorithms. $y$ axis is $n$ and $x$ axis is $\tau$. Black means fail.}
        \label{fig:success_rate}
   \end{figure*}

    \begin{figure*}[htbp]
        \begin{minipage}[b]{0.33\linewidth}
            \centering
            \includegraphics[width=0.9\linewidth]{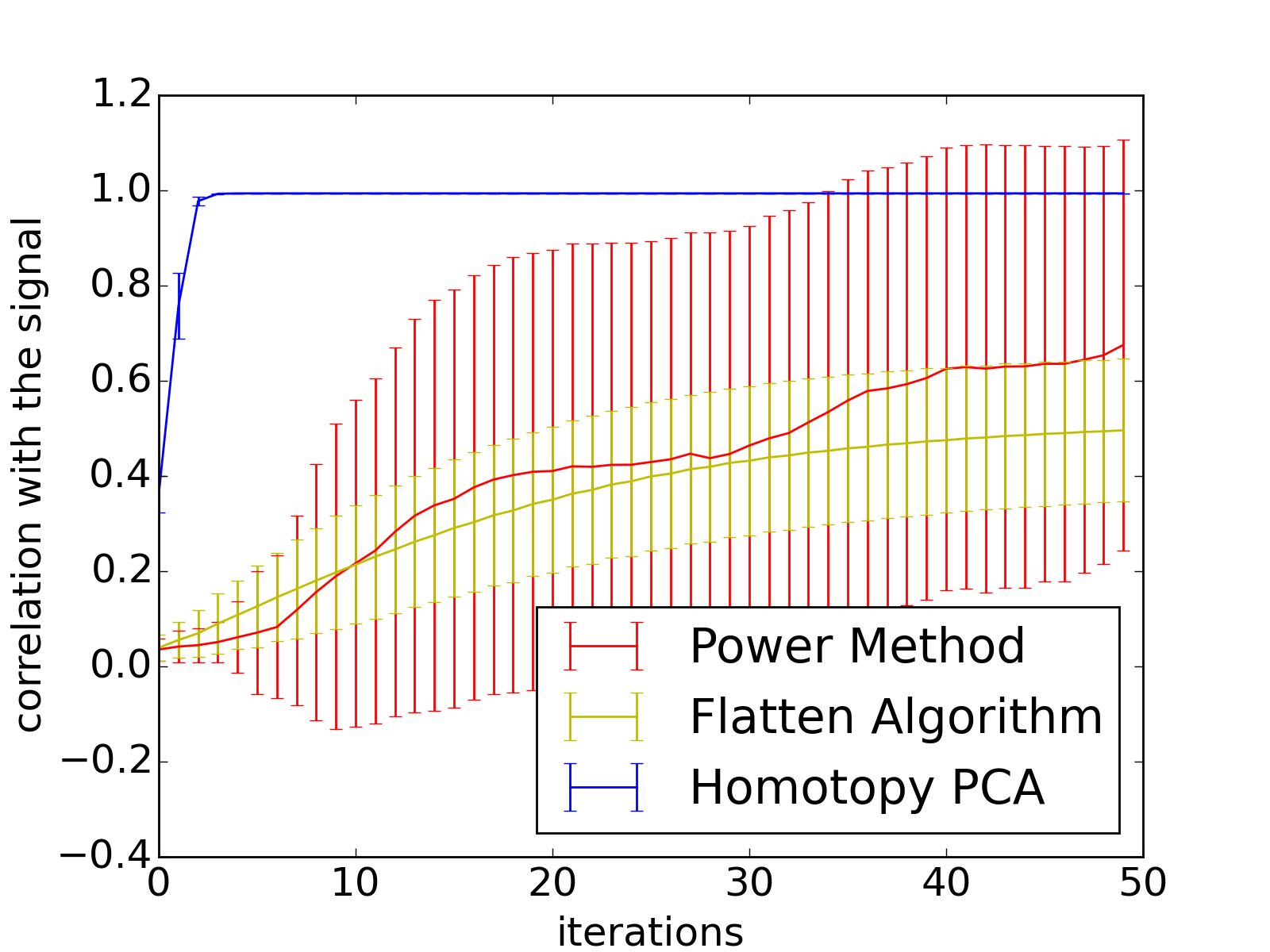}
            {$\alpha = 1.1$}
        \end{minipage}
        \begin{minipage}[b]{0.33\linewidth}
            \centering
            \includegraphics[width=0.9\linewidth]{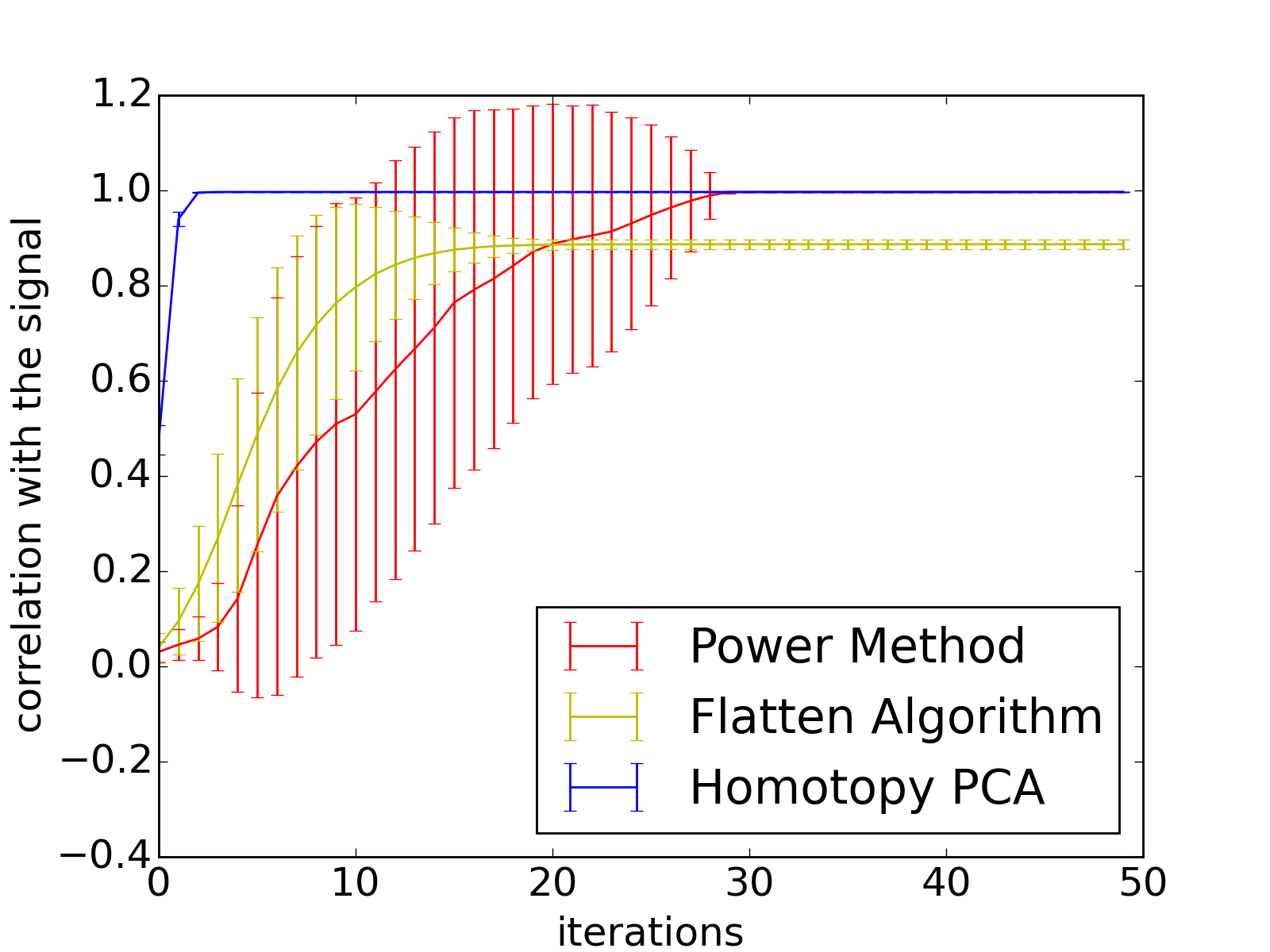}
            {$\alpha = 1.5$}
        \end{minipage}
        \begin{minipage}[b]{0.33\linewidth}
            \centering
            \includegraphics[width=0.9\linewidth]{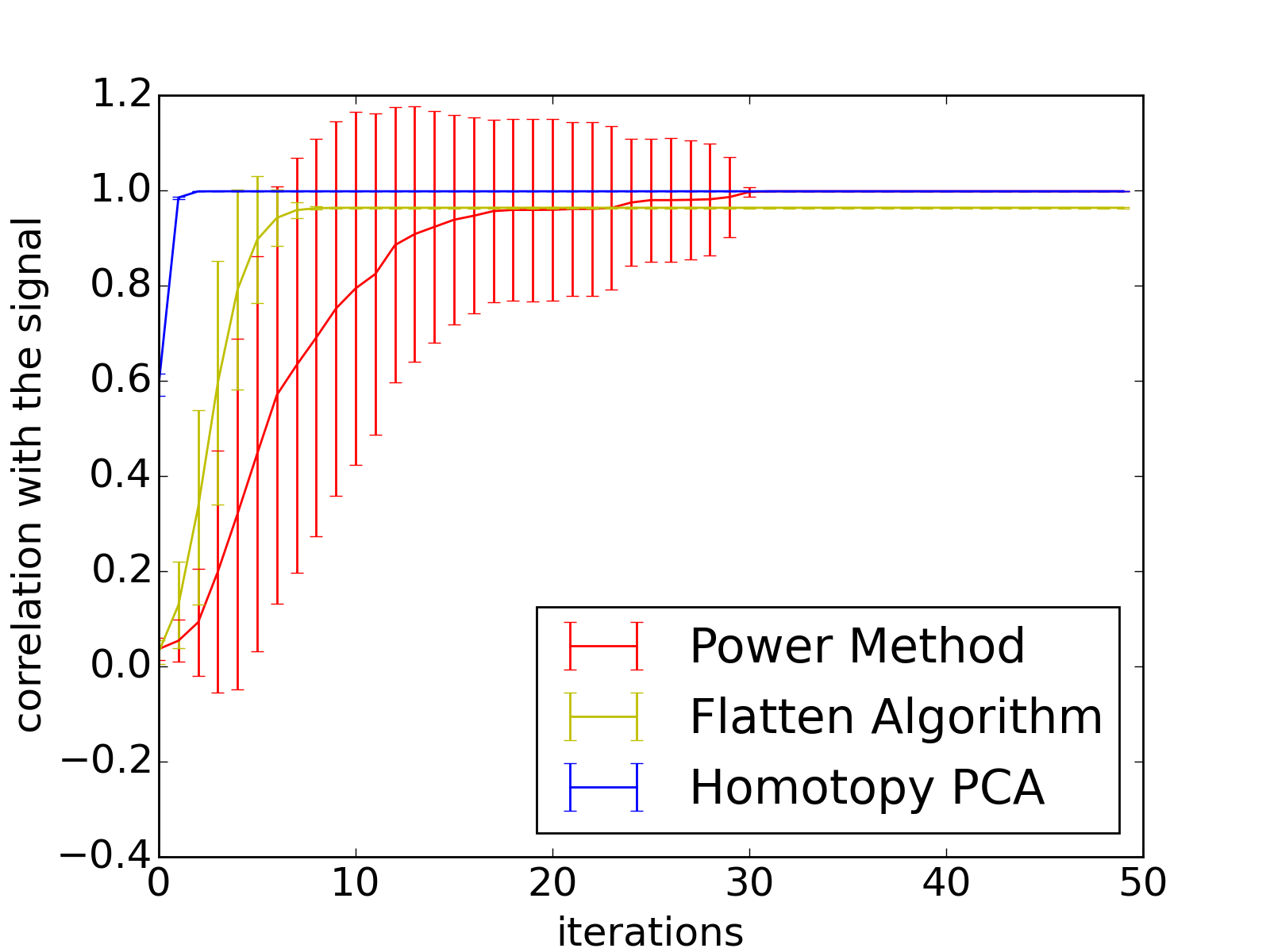}
            {$\alpha = 2$}
        \end{minipage}
        \caption{Rate of Convergence. $\tau = \alpha n^{\frac 3 4}$, $x$ axis is the number of iterations, $y$ axis is the expected correlation with signal $\bm v$ (with variance represented as error bars)}
        \label{fig:iteration_comparison}
    \end{figure*}

For brevity we refer to our Tensor PCA with homotopy initialization method (Algorithm~\ref{algo:simple}) as HomotopyPCA. We compare that with two other algorithms: the Flatten algorithm and the Power method. The Flatten algorithm was originally proposed by \cite{richard2014statistical}, where they show it works when $\tau = \Omega(n)$. \cite{hopkins2015tensor} accelerated the Flatten algorithm to near-linear time, and improved the analysis to show it works when $\tau = \tilde{\Omega}(n^{3/4})$. The Power method is similar to our algorithm, except it does not use intuitions from homotopy, and initialize at a random vector. Note that there are other algorithms proposed in \cite{hopkins2015tensor}, however they are based on the Sum-of-Squares SDP hierarchy, and even the fastest version runs in time $O(n^5)$ (much worse than the $O(n^3)$ algorithms compared here).

We first compare how often these algorithms successfully find the signal vector $v$, given different values of $\tau$ and $n$. The results are in Figure~\ref{fig:success_rate}, in which $y$-axis represents $n$ and $x$-axis represents $\tau$. We run 50 experiments for each values of $(n,\tau)$, and the grayness in each grid shows how frequent each algorithm succeeds: black stands for ``always fail'' and white stands ``always succeed''. For every algorithm, we say it fails if (1) when it converges, i.e., the result at two consecutive iterations are very close, the correlation with the signal $\bm v$ is less than $80\%$; (2) the number of iterations exceeds 100. In the experiments for Power Method, we observe there are many cases where situation (1) is true, although our new algorithms can always find the correct solution. In these cases the function indeed have a local maximizer. From Figure~\ref{fig:success_rate}, our algorithm outperforms both Power Method and the Flatten algorithm in practice. This suggests the constant hiding in our algorithm is possibly smaller. 

Next we compare the number of iterations to converge with $n = 500$ and $\tau = \alpha n^{\frac 3 4}$, where $\alpha$ varies in $[1.1, 1.5, 2]$. In Figure~\ref{fig:iteration_comparison}, the x-axis is the number of iterations, and the $y$ axis is the correlation with the signal $\bm v$ (error bars shows the distribution from 50 independent runs). For all $\alpha$, Homotopy PCA performs well --- converges in less than $5$ iterations and finds the signal $\bm v$. The Power Method converges to a result with good correlations with the signal $\bm v$, but has large variance because it sometimes gets trapped in local optima. As for the Flatten algorithm, the algorithm always converges. However, it takes more iterations compared to our algorithm. Also when $\alpha$ is small, the converged result has bad correlation with $\bm v$.

\bibliographystyle{unsrt}
\bibliography{bib}

\newpage
\appendix

\section{Omitted Proofs}

\subsection{Omitted Proof in Section~\ref{sec:prelim}}\label{app:lemma1}

\begin{lemma} [Lemma~\ref{lem:smooth} restated]
$$ g(\bm x, t) = \tau \langle \bm v, \bm x \rangle^3 + t^2 \langle 3 \tau \bm v + \bm u, \bm x \rangle + \bm A(\bm x,\bm x,\bm x),$$
where the vector $\bm u$ is defined by $\bm u_j = \sum_{i=1}^n (\bm A_{iij}+\bm A_{iji} + \bm A_{jii})$.
Moreover, let $\bm z$ be a vector where $\bm z_j = \sum_{i=1}^d(\bm T_{iij}+\bm T_{iji}+\bm T_{jii})$, then we have $\bm z = 3\tau \bm v+\bm u$.
\end{lemma}

\begin{proof}
We can write $g(x,t)$ as an expectation
    \begin{align*}
        g(\bm x, t) = \int_{\R^n} f(\bm x + \bm y) k_{t}(\bm y) dy & = \mathbb E_{y \sim N(\bm 0, t^2 \mathbf I_n)} [f(\bm x + \bm y)] = \mathbb E_{y \sim N(\bm 0, \mathbf I_n)} [ f(\bm x + t\bm y)]
    \end{align*}
    Since $f$ is just a degree 3 polynomial, we can expand it and use the lower moments of Gaussian distributions:
    \begin{align*}
        g(\bm x, t) &~= \mathbb E [f(\bm x + t \bm y)] \\
        &~= \mathbb E [\tau \langle \bm v, (\bm x + \bm t y) \rangle^3 + \bm A(\bm x + t \bm y, \bm x + t \bm y, \bm x + t \bm y)] \\
        &~= \tau \langle \bm v, \bm x \rangle^3 + 3 \tau t^2 \langle \bm v, \bm x \rangle \cdot \mathbb E[\langle \bm v, \bm y \rangle^2] +  \mathbb E[\bm A(\bm x + t \bm y, \bm x + t \bm y, \bm x + t \bm y)] \\
        &~= \tau \langle \bm v, \bm x \rangle^3 + 3 \tau t^2 \langle \bm v, \bm x \rangle + t^2 \sum_{i,j} (\bm A_{iij} + \bm A_{iji} + \bm A_{jii}) \bm x_j + \bm A(\bm x,\bm x,\bm x)  
    \end{align*}
    Therefore the first part of the lemma holds if we define $\bm u$ to be the vector $\bm u_j = \sum_i \bm A_{iij} + \bm A_{iji} + \bm A_{jii}$. 
    In order to compute the vector $3\tau \bm v+\bm u$, notice that the term $\inner{3\tau \bm v+\bm u, \bm x}$ is the linear term on $\bm x$, and it is equal to
    $$\inner{3\tau \bm v+\bm u, \bm x} = \E_{y\sim N(\bm 0,\mathbf I_n)}[\bm T(\bm x,\bm y,\bm y)+\bm T(\bm y,\bm x,\bm y) + \bm T(\bm y,\bm y, \bm x)].$$
     This means 
     $$(3\tau \bm v+\bm u)_j = \E_{y\sim N(\bm 0,\mathbf I_n)}[\bm T(\bm e_j,\bm y,\bm y)+\bm T(\bm y,\bm e_j,\bm y) + \bm T(\bm y,\bm y, \bm e_j)] = \sum_{i=1}^d (\bm T_{iij}+\bm T_{iji}+\bm T_{jii}).$$
\end{proof}

\subsection{Omitted Proof in Section~\ref{sec:simplealg}}\label{app:theorem3}

    \begin{theorem}[Theorem~\ref{thm:converge} restated]%        For any $t = o(n^{-\frac 1 2})$, we can apply a gradient descent algorithm on $g(\bm x, t)$ with $\bm x^{\infty}$ as the initial solution and infinite-step update. Such algorithm converges to a vector close to $\bm v$ in $O(\frac{\log n}{c \log \log n})$ iterations.
    When $\tau \ge Cn^{3/4}\log n$ for a large enough constant $C$, under the Independence Condition (Condition~\ref{assump}), Algorithm~\ref{algo:simple} finds a vector $\bm x^m$ such that $\inner{\bm x^m, \bm v} \geq 0.8$ in $O(\log \log n)$ iterations.
    \end{theorem}

    \begin{algorithm2e}[h] 
        \KwIn{Tensor $\bm T = \tau \cdot v^{\otimes 3} + \bm A$;}
        \KwOut{Approximation of $\bm v$;}
        $m = O(\log \log n)$;\\

        $\forall~j, \bm x^0_j = \sum_i \bm T_{iij} + \bm T_{iji} + \bm T_{jii}$;\\
        $\bm x^0 = \bm x^0 / \|\bm x^0\|$; \hfill//$\bm x^0 = \bm x^\dag$\\
        \For {$k = 0$ to $m$} {
            $\bm x^{k+1} = \bm T(\bm x^{k},\bm x^{k},:) + \bm T(\bm x^k,:,\bm x^k) + \bm T(:,\bm x^k,\bm x^k)$; \\
            $\bm x^{k+1} = \bm x^{k+1} / \|\bm x^{k+1}\|$;
        }
        {\Return $\bm x^m$};
        \caption{Tensor PCA by Homotopy Initialization}
    \end{algorithm2e}   

    \begin{proof}
        We first show the initial maximizer $\bm x^0$ already has a nontrivial correlation with $\bm v$. Recall $\bm x^0 = \frac{3\tau \bm v+\bm u}{\|3\tau \bm v+\bm u\|_2}$. Note that if $\tau$ is very large such that $\|3\tau \bm v\|_2 \ge 10 \|u\|_2$, then we already have $\inner{\bm x^0, \bm v} \ge 0.8$. Later we will show that whenever $\inner{\bm x^i, \bm v} \ge 0.8$ all later iterations have the same property.
        
%         We can bound $\|3 \tau \bm v + \bm u\|_2$ from above and below by $ \|\bm u\|_2 \pm \|3 \tau \bm v\|_2$. Notice, the latter obeys $\Theta(n \sqrt m) \pm 3 \Theta(n^{\frac 3 4} \cdot \log^c n)$. If $\tau$ is large enough so that the term $\|3 \tau \bm v\|_2$ dominates, then $\langle \bm x^0, \bm v \rangle = \Theta(1)$, implying $\bm x^0$ already has high correlations with $\bm v$. 
Therefore, we are left with the case when $\|\bm u\|_2 \ge 0.1 \|3\tau \bm v\|_2$ (this implies $\tau \le O(n)$). In this case, by Condition~\ref{assump} we know $|\inner{\bm u, \bm v}| = O(\sqrt{nm\log n})$ and $\|\bm u\|_2 = O(n\sqrt{m})$, therefore 
    \begin{align*}
        \|3 \tau \bm v + \bm u\|_2 \in \left[\sqrt{\|\bm u\|_2^2+\|3\tau \bm v\|_2^2 - O(\tau \sqrt{nm\log n})}, \sqrt{\|\bm u\|_2^2+\|3\tau \bm v\|_2^2 + O(\tau \sqrt{nm\log n})}\right]    
    \end{align*}
    Therefore, $\|3 \tau \bm v + \bm u\|_2 = \Theta(n\sqrt{m})$. Assume $\tau \ge Cn^{3/4}\log^c n$ for large enough $C$ (where we will later show $c = 1$ suffices)
        \begin{align*}
                \langle \bm x^0, \bm v \rangle = \frac{3 \tau + \langle \bm u, \bm v \rangle} {\|3 \tau \bm v + \bm u\|_2} = \frac 1 {O(n \sqrt m)} \Theta(n^{\frac 3 4} \cdot \log^c n) \ge \frac{n^{-\frac 1 4} \cdot \log^c n} {\sqrt m}.
        \end{align*}
        
        Now let us consider the first step of power method.
%        Consider the gradient of $g(\bm x^p, t)$:
%        \[
%            \nabla g(\bm x, t) = 3 \tau \langle \bm v, \bm x^p \rangle^2 \bm v + t^2 (3 \tau \bm v + \bm u) + \delta(\bm x^p).
%        \]
%        With high probability, we have $t^2 \|3 \tau \bm v + \bm u\|_2 = t^2 \Theta(n \sqrt m)$.
        Let $\hat{\bm x}^1$ be the vector before normalization. Observe that  $\hat{\bm x}^1 = 3 \tau \langle \bm v, \bm x^0 \rangle^2 \bm v + \delta(\bm x^0)$. By Condition~\ref{assump} we have bounds on $\|\delta(\bm x^0)\|$ and $|\inner{\delta(\bm x^0),v}|$, therefore we have 
       \begin{align*}
            \langle \hat{\bm x}^{1}, \bm v \rangle = 3 \tau \langle \bm v, \bm x^0 \rangle^2 + \inner{\delta(\bm x^0),\bm v} \in \left[3 \tau \langle \bm v, \bm x^0 \rangle^2 - O(\sqrt{m \log n}), 3 \tau \langle \bm v, \bm x^0 \rangle^2 + O(\sqrt{m \log n})\right].
       \end{align*}
        Note that when $\tau \ge Cn^{3/4}\log^c n$ and $\log^c n \gg m$, the first term is much larger than $\sqrt{m\log n}$. Hence for the first iteration, we have $\langle \hat{\bm x}^{1}, \bm v \rangle \ge (3-o(1))\tau \inner{\bm v,\bm x^0}^2\ge 2Cn^{\frac 1 4} \cdot \log^{3c} n / m$.

Similar as before, when $\|\delta(\bm x^0)\|_2 \le 0.1 \|3\tau \inner{v,\bm x^0}^2 v\|_2$, we already have $\inner{\bm x^1, \bm v} \ge 0.8$. On the other hand, if $\|\delta(\bm x^0)\|_2 \ge 0.1\|3\tau \inner{v,\bm x^0}^2 v\|_2$, in this case, by Condition~\ref{assump} we know $\|\delta(\hat{\bm x}^0)\| = O(\sqrt{nm})$. We again have $\|\hat{\bm x}^1\|_2 \in \sqrt{\|\delta(\bm x^0)\|_2^2+\|3\tau \inner{\bm v, \bm x^0}^2 \bm v\|_2^2 \pm O(\tau \inner{\bm v, \bm x^0}^2 \sqrt{nm})}$. Therefore, $\|\hat{\bm x}^1\|_2 = O(\sqrt {n m})$.
    Combining the bounds for the norm of $\hat{\bm x}^1$ and its correlation with $\bm v$,
        \[
            \langle \frac{\hat{\bm x}^{1}}{\|\hat{\bm x}^{1}\|}, \bm v \rangle \ge n^{-\frac 1 4} \cdot \log^{3c} n / m^{\frac 3 2}.
        \]
        Therefore, when $\log^c n \gg m$, the correlation between $\bm x^1$ and $\bm v$ is larger than the correlation between $\bm x^0$ and $\bm v$. This shows the first step makes an improvement.
        
%        Next, we partition the gradient into two parts, $h(\bm x^p, t)$ and $t^2 (3 \tau \bm v + \bm u)$. According to previous calculations, $h(\bm x^p, t)$ points to the direction of $\bm v$. Moreover, when $t = o(n^{-\frac 1 2})$, $\|h(\bm x^p, t)\|_2$ dominates $\|t^2 (3 \tau \bm v + \bm u)\|_2$, and thus, the gradient is dominated by $h(\bm x^p, t)$. Moreover, since $\log^c n \gg m$, the gradient has better correlation with $\bm v$ than $\bm x^p$.

In order to show this for the future steps, we do induction over $p$. The induction hypothesis is for every $p$, either $\inner{\bm x^p, \bm v} \ge 0.8$ or
        \[
            \langle \bm x^p, \bm v \rangle \ge n^{-\frac 1 4} \log^{3^p c} n / m^{2^p - \frac 1 2}.
        \]
        Initially, for $p = 0$, we have already proved the induction hypothesis. 
        
        Now assume the induction hypothesis is true for $p$. In the next iteration, let $\hat{\bm x}^{p+1}$ be the vector before normalization. Similar as before we have $\hat{\bm x}^{p+1} = 3 \tau \langle \bm v, \bm x^p \rangle^2 \bm v + \delta(\bm x^p)$.

When $\inner{\bm x^p,\bm v} \ge 0.8$, by Condition~\ref{assump} we know the norm of $\|\delta(\bm x^p)\|$ is much smaller than $3\tau \inner{\bm x^p,\bm v}^2$. Therefore we still have $\inner{\bm x^{p+1},\bm v} \ge 0.8$. 
        
      In the other case, we follow the same strategy as the first step. By Condition~\ref{assump} we can compute the correlation between $\hat{\bm x}^{p+1}$ and $\bm v$:
        \begin{align*}
            \langle \hat{\bm x}^{p+1}, \bm v \rangle &= 3 \tau \langle \bm v, \bm x^p \rangle^2 \pm O(\sqrt{m \log n}) \\ &\ge 2Cn^{\frac 1 4} \log^{3^{p+1}c} n / m^{2^{p+1} - 1}.
        \end{align*}
For the norm of $\hat{\bm x}^{p+1}$, notice that the first term $3 \tau \langle \bm v, \bm x^p \rangle^2 \bm v$ has norm $3 \tau \langle \bm v, \bm x^p \rangle^2$, and the second term $\delta(\bm x^p)$ has norm $\Theta(\sqrt{nm})$. Note that these two terms are almost orthogonal by Independence Condition, therefore
        \[
            \|\hat{\bm x}^{p+1}\|_2 = \Theta(\tau \langle \bm v, \bm x^p \rangle^2) + O(\sqrt {n m})
        \]
        If $3\tau \langle \bm v, \bm x^p \rangle^2 \ge \Delta \sqrt {n m}$, then $\|\hat{\bm x}^{p+1}\|_2 \le (3+\Delta')\tau \langle \bm v, \bm x^p \rangle^2$, where $\Delta'$ is a constant that is smaller than $0.1$ when $\Delta$ is large enough. Therefore in this case $\langle \frac{\hat{\bm x}^{p+1}}{\|\hat{\bm x}^{p+1}\|_2}, \bm v \rangle \ge 0.8$. Thus we successfully recover $\bm v$ in the next step.
        
         Otherwise, we know $\|\hat{\bm x}^{p+1}\|_2 = O(\sqrt {n m})$. Then,
        \[
            \langle \frac{\hat{\bm x}^{p+1}}{\|\hat{\bm x}^{p+1}\|_2}, \bm v \rangle \ge n^{-\frac 1 4} \cdot \log^{3^{p+1}c} n / m^{2^{p+1} - \frac 1 2}
        \]
        If we select $c = 1$, after $m = O(\log \log n)$ rounds, we have $\inner{\bm x^p, \bm v} \ge n^{-\frac 1 4} \log^{3^p c} n / m^{2^p - \frac 1 2} \ge 0.8$, therefore we must always be in the first case. As a result $\langle \bm x^m, \bm v \rangle \ge 0.8$. 
    \end{proof}

    \begin{lemma}[Lemma \ref{lem:resample} restated]
        Let the sequence $\bm T^0, \cdots, \bm T^{m-1}$ be generated according to Section~\ref{sec:resampling}. Let $\bm Q^i = \tau \bm v^{\otimes 3}+\bm C^i$, where $\bm C^i$'s are tensors with independent Gaussian entries. Each entry in $\bm C^i$ is distributed as $N(0,m)$. The two sets of variables $\{\bm T^i\}$ and $\{\bm Q^i\}$ has the same distribution. 
        %as i.i.d Gaussians with mean $\tau \cdot v^{\otimes 3}$ and variance $m \bm J$, where $\bm J$ is an all-one third-order tensor.
    \end{lemma}
    \begin{proof}
    Note that both distributions are multivariate Gaussians. Therefore we only need to show that they have the same first and second moments.    
    
    For the first moment, this is easy, we have $\mathbb E [\bm T^p] = \tau \cdot v^{\otimes 3}$ and $\mathbb E [\bm Q^p] = \tau \cdot v^{\otimes 3}$ for all $p$.
    
    For the second moment (covariance), we consider the covariance between $T^p_{ijk}$ and $T^q_{i'j'k'}$. Note that for the distribution $Q$, as long as the 4 tuple $(p,i,j,k)\ne (q,i',j',k')$ the correlation is 0. We first show when $(i,j,k) \neq (i',j',k')$ we have
    \begin{align*}
         \mathrm{Cov}(\bm T^p_{ijk}, \bm T^q_{i'j'k'}) 
        = & ~\mathbb E [(\bm T^p_{ijk} - \tau \bm v_i \bm v_j \bm v_k) (\bm T^q_{i'j'k'} - \tau \bm v_{i'} \bm v_{j'} \bm v_{k'})] \\
        = & ~\mathbb E [(\bm B^p_{ijk} - \overline{\bm B_{ijk}} + \bm A_{ijk}) (\bm B^q_{i'j'k'} - \overline{\bm B_{i'j'k'}} + \bm A_{i'j'k'})] \\
        = & ~\mathbb E [\bm B^p_{ijk} - \overline{\bm B_{ijk}} + \bm A_{ijk}] \mathbb E [\bm B^q_{i'j'k'} - \overline{\bm B_{i'j'k'}} + \bm A_{i'j'k'}] \\
        = & ~0 \\
    \end{align*}
Hence for these variables the two distributions have the same covariance.
    
    Next we consider the case $p \neq q$, 
    \begin{align*}
        \mathrm{Cov}(\bm T^p_{ijk}, \bm T^q_{ijk})
        = & ~\mathbb E [(\bm T^p_{ijk} - \tau \bm v_i \bm v_j \bm v_k) (\bm T^q_{ijk} - \tau \bm v_i \bm v_j \bm v_k)] \\
        = & ~\mathbb E [(\bm B^p_{ijk} - \overline{\bm B_{ijk}} + \bm A_{ijk}) (\bm B^q_{ijk} - \overline{\bm B_{ijk}} + \bm A_{ijk})] \\
        = & ~-\frac {m - 1} {m^2} \mathbb E[(\bm B^p_{ijk})^2 + (\bm B^q_{ijk})^2] + \sum_{l \neq p, q} \frac {1} {m^2} \mathbb E[(\bm B^l_{ijk})^2] + \mathbb E[\bm A_{ijk}^2] \\
        = & ~-\frac {2(m - 1)} {m^2} \cdot m + \frac {m - 2} {m^2} \cdot m + 1 = 0
    \end{align*}
    The covariance for these entries also match.
    
    Finally we need to consider the variance for each entry of $\bm T^p$ and $\bm Q^p$. To do that we compute the Variance of $\bm T^p_{ijk}$
    \begin{align*}
        \mathrm{Var}(\bm T^p_{ijk})
        = & ~\mathbb E [(\bm T^p_{ijk} - \tau \bm v_i \bm v_j \bm v_k) (\bm T^p_{ijk} - \tau \bm v_i \bm v_j \bm v_k)] \\
        = & ~\mathbb E [(\bm B^p_{ijk} - \overline{\bm B_{ijk}} + \bm A_{ijk}) (\bm B^p_{ijk} - \overline{\bm B_{ijk}} + \bm A_{ijk})] \\
        = & ~\frac {(m - 1)^2} {m^2} \mathbb E[(\bm B^p_{ijk})^2] + \sum_{l \neq p} \frac {1} {m^2} \mathbb E[(\bm B^l_{ijk})^2] + \mathbb E[\bm A_{ijk}^2] \\
        = & ~\frac {(m - 1)^2} {m^2} \cdot m + \frac {m - 1} {m^2} \cdot m + 1 = m
    \end{align*}
    This is also the same as the variance of $Q^p_{ijk}$. Therefore the two multivariate Gaussians have the same mean and covariance, and must be the same distribution.
        \end{proof}

\begin{lemma}[Lemma~\ref{lem:resamplingassump} restated] Let $\bm T^p$ be generated according to Algorithm~\ref{algo:resample} and $\bm A^p = \bm T^p - \tau \bm v^{\otimes 3}$. Let $\bm u^0$ be a vector such that $\bm u^0_j = \sum_i \bm A_{iij}^0 + \bm A_{iji}^0 + \bm A_{jii}^0$, and $\delta^p(\bm x^p) = \bm A^p(\bm x^p, \bm x^p, :)+\bm A^p(\bm x^p, :, \bm x^p)+\bm A^p(:, \bm x^p, \bm x^p)$. With high probability, (1) $\|\bm u^0\|_2 = \Theta(n \sqrt{m})$ and $|\langle \bm u^0, \bm v \rangle| = O(\sqrt {n m \log n})$; (2) for the sequence computed by Algorithm \ref{algo:resample}, $\bm x^0, \bm x^1, \cdots, \bm x^{m-1}$, $\forall~ 0 \leq p \leq m-1$, $\|\delta^p(\bm x^p)\|_2 = \Theta(\sqrt{n m}) \|\bm x^p\|_2^2$ and $|\langle \delta^p(\bm x^p), \bm v \rangle| = O(\sqrt{m \log n}) \|\bm x^p\|_2^2$. As a result Condition~\ref{assump} is satisfied.
\end{lemma}

%Note that we only need $\bm u^0$ to satisfy the constraint.
\begin{proof}

Since by Lemma~\ref{lem:resample}, we know the noise tensors $\bm A^p$ used in $p$-th step behave exactly the same as independent Gaussian tensors. The vectors $\bm u^0$ and $\delta(x^p)$ are therefore spherical Gaussian random variables conditioned on any value of $\bm x^i$. Therefore we can prove this lemma by standard Gaussian concentration results.
\begin{claim}\label{clm:gaussian}\citep{laurent2000adaptive}
Suppose $\bm x$ is a $d$-dimensional spherical Gaussian, then
$$
\Pr[|\|\bm x\|^2 - \E[\|\bm x\|^2]| \ge \frac{1}{2}\E[\|\bm x\|^2]] \le e^{-\Omega(d)}.
$$
Also, for any fixed vector $\bm v$, $\inner{\bm x, \bm v}$ is also a Gaussian distribution that satisfies
$$
\Pr[|\inner{\bm x,\bm v}| \ge t\sqrt{\E[\inner{\bm x,\bm v}^2]}] \le e^{-\Omega(t^2)}.
$$
\end{claim}

 For terms like $\|\bm u^p\|$ and $\|\delta(\bm x^p)\|$, we know the norm of a Gaussian random variable obeys the $\chi^2$ distribution and is highly concentrated to its expectation. For terms like $\inner{\bm u^p,\bm v}$ and $\inner{\delta(\bm x^p), \bm v}$, we know they are just Gaussian distributions and is always bounded by $O(\sigma\sqrt{\log n})$ with high probability. Therefore we only need to compute the expected norms of these vectors.

    \begin{align*}
        \mathbb E[\|\bm u^p\|_2^2] &~= \mathbb E[\sum_j (\sum_i \bm A^p_{iij} + \bm A^p_{iji} + \bm A^p_{jii})^2] \\
        &~= \mathbb E[\sum_{j} (\sum_{i \neq j} (\bm A^p_{iij})^2 + (\bm A^p_{iji})^2 + (\bm A^p_{jii})^2) + 9 (\bm A^p_{jjj})^2] \\
        &~= 3 n (n-1) m + 9 n m \\
        &~= \Theta(n^2 m)
    \end{align*}
    
Therefore by Claim~\ref{clm:gaussian} we have $\|u\|_2 = \Theta(n\sqrt{m})$ with high probability.    
    
    \begin{align*}
        \mathbb E[\langle \bm u^p, \bm v \rangle] = \mathbb E[\sum_j (\sum_i \bm A^p_{iij} + \bm A^p_{iji} + \bm A^p_{jii}) \bm v_j] = 0
    \end{align*}

    \begin{align*}
        \mathbb E[\langle \bm u^p, \bm v \rangle^2] &~= \mathbb E[\sum_j ((\sum_i \bm A^p_{iij} + \bm A^p_{iji} + \bm A^p_{jii}) \bm v_j)^2] \\
        &~= \mathbb E[\sum_j \bm v_j^2 (9 (\bm A^p_{jjj})^2 + \sum_{i \neq j} (\bm A^p_{iij})^2 +  (\bm A^p_{iji})^2 + (\bm A^p_{jii})^2)] \\
        &~=9 m + 3(n - 1) m \\
        &~=\Theta(n m)
    \end{align*}
    
    This means $\inner{u,v}$ is a Gaussian random variable with variance $\sigma^2 = \Theta(nm)$, therefore for any constant $C'$, with probability at least $1-n^{-C'}$ we know $|\inner{u,v}| \le O(\sqrt{nm \log n})$. We can apply union bound over all $p$ and get the desired result.
    
    Similarly we can compute the expected square norm of $\delta(\bm x^p)$ as below
    
    \begin{align*}
        \mathbb E[\|\delta(\bm x^p)\|_2^2] &~= \Theta(1) \mathbb E[\|\bm A^p(\bm x^p, \bm x^p, :)\|_2^2] \\
        &~= \Theta(1) \mathbb E[\sum_k (\sum_{i,j} \bm A^p_{ijk} \bm x^p_i \bm x^p_j)^2] \\
        &~= \Theta(1) \mathbb E[\sum_k (\sum_{i,j} (\bm A^p_{ijk})^2 (\bm x^p_i)^2 (\bm x^p_j)^2)] \\
        &~= \Theta(1) n m \|x^p\|_2^4
    \end{align*}
    \begin{align*}
        & \mathbb E[\langle \delta(\bm x^p), \bm v \rangle] = \sum_{i,j,k} \mathbb E[\bm A^p_{ijk} (\bm x^p_i \bm x^p_j \bm v_k + \bm x^p_i \bm v_j \bm x^p_k + \bm v_i \bm x^p_j \bm x^p_k)] = 0
    \end{align*}
    \begin{align*}
        \mathbb E[\langle \delta(\bm x^p), \bm v \rangle^2] &~= \sum_{i,j,k} \mathbb E[(\bm A^p_{ijk})^2 (\bm x^p_i \bm x^p_j \bm v_k + \bm x^p_i \bm v_j \bm x^p_k + \bm v_i \bm x^p_j \bm x^p_k)^2] \\
        &~= 3 m \sum_k \bm v_k^2 \sum_{i,j} (\bm x_i^p)^2 (\bm x_j^p)^2 + 6 m \sum_i (\bm x_i^p)^2 \sum_{j,k} \bm v_i \bm v_j \bm x_j^p \bm x_k^p \\
        &~= 3 m \|\bm x^p\|_2^4 + 6 \|\bm x^p\|_2^2 \langle \bm v, \bm x^p \rangle^2 \\
        &~\leq 9 m \|\bm x^p\|_2^4
    \end{align*}
    
    The bounds on $\|\delta(\bm x^p)\|$ and $\inner{\delta(\bm x^p),v}$ follows immediately from these expectations.
%    Thus, our construction creates a series of distribution that meets our condition and henceforth, we finish the proof of the main theorem.
\end{proof}

\subsection{Omitted Proof in Section~\ref{sec:path}}\label{app:path}

    \begin{lemma}[Lemma~\ref{lem:smoothalternative} restated]
    The smoothed version of the alternative objective is
    \begin{align*}
        g_r(\bm x, t) = \tau \langle \bm v, \bm x \rangle^3 + t^2 \langle 3 \tau \bm v + \bm u, \bm x \rangle + \bm A(\bm x,\bm x,\bm x) - \frac {3 \tau} 4 \left(\|x\|_2^4 + 2 t^2 (n+2) \|x\|_2^2 + t^4 (n^2+2n)\right)
    \end{align*}
    Its gradient and Hessian are equal to
\begin{align*}
        \nabla g_r(\bm x, t) = 3 \tau \langle \bm v, \bm x \rangle^2 \bm v + t^2 (3 \tau \bm v + \bm u) 
         + \delta(\bm x) - 3 \tau (\|\bm x\|_2^2 \bm x + t^2 (n+2) \bm x).
    \end{align*}
    and
    \begin{align*}
        \nabla^2 g_r(\bm x, t) =  -3 \tau ((\|\bm x\|_2^2 + t^2 (n+2)) \bm I - 2 \langle \bm v, \bm x \rangle \bm v \bm v^T + 2 \bm x \bm x^T) + P_{sym}[\bm A(\bm x, :, :) + \bm A(:, \bm x, :) + \bm A(:, :, \bm x)].
    \end{align*}    
    \end{lemma}

\begin{proof}
    Similar to Lemma~\ref{lem:smooth}, we can write the smoothing operation as an expectation. By linearity of expectation we know 
$$
g_r(\bm x,t) = g(\bm x, t)+\mathbb E[\|\bm x + t \bm y\|_2^4]
$$    
    
  We can compute the new terms by the moments of Gaussians:
    \begin{align*}
        \mathbb E[\|\bm x + t \bm y\|_2^4] & = \E[(\|\bm x\|_2^2+2t\inner{\bm x, \bm y} + t^2\|\bm y\|_2^2)^2] \\
        & = \E[\|\bm x\|_2^4+4t^2\inner{\bm x, \bm y}^2 + t^4\|\bm y\|_2^4 + 2t^2\|x\|^2\|y\|^2] \\
        & = \|x\|_2^4 + t^2 (2 n + 4) \|x\|_2^2 + t^4 (n^2+2n) = \|x\|_2^4 + 2 t^2 (n+2) \|x\|_2^2 + t^4 (n^2+2n).
    \end{align*}
    
    Here in the second equation we omitted all the odd order terms for $\bm y$ because those terms have expectation 0. The final step uses the moments of Gaussians.
    
    The equation for $g_r(\bm x,t)$ follows immediately, and since it is a polynomial it is easy to compute its gradient and Hessian.
\end{proof}

Before trying to characterize the local maxima on the homotopy path, let us first prove the following property for the matrix $P_{sym}[\bm A(\bm x, :, :) + \bm A(:, \bm x, :) + \bm A(:, :, \bm x)]$.

\begin{lemma}
Let $H(\bm x) = P_{sym}[\bm A(\bm x, :, :) + \bm A(:, \bm x, :) + \bm A(:, :, \bm x)]$, there exists constants $c^-,c^+$ such that with probability at least $1-\exp(-\Omega(n))$, for any unit vector $\bm x$ we have
$$
c^-\sqrt{n} \le \lambda_{max} \le c^+\sqrt{n}.
$$
\label{lem:hessianprop}
\end{lemma}

\begin{proof}
For the upperbound, we use the bound on tensor spectral norm. \citet{tomioka2014spectral} proved that for a random Gaussian tensor $\bm A$, with probability at least $1-\exp(-\Omega(n))$ we know for any vectors $\bm x, \bm y, \bm z$, $|\bm A(\bm x, \bm y, \bm z)| \le O(\sqrt{n})$. Therefore for any unit vector $\bm y$, $|\bm y^\top H(\bm x) \bm y| = |\bm A(\bm x, \bm y, \bm y) + \bm A(\bm y, \bm x, \bm y) + \bm A(\bm y, \bm y, \bm x)| \le O(\sqrt{n})$. 

For the lowerbound, we use the distribution of the largest eigenvalue of Gaussian Orthogonal Ensemble. Suppose $\bm M$ is a random matrix whose entries are i.i.d. standard Gaussians, then the symmetric matrix $\frac{\bm M + \bm M^\top}{\sqrt{2}}$ is distributed according to the Gaussian Orthogonal Ensemble. Let $P_{\bm x^\perp}$ be the projection operator to the orthogonal subspace of $\bm x$, then the key observation is $P_{\bm x^\perp} H(\bm x) P_{\bm x^\perp}$ is (up to a constant scaling) distributed as a Gaussian Orthogonal Ensemble of dimension $(n-1)\times (n-1)$. To see this, the easiest way is to observe that Gaussians are invariant under rotation, so we can take $\bm x = \bm e_1$. Now for $i,j = \{2,3,...,n\}$, $[P_{\bm x^\perp} H(\bm x) P_{\bm x^\perp}]_{i,j} = \bm A_{1ij}+\bm A_{1ji}+\bm A_{i1j}+\bm A_{j1i}+\bm A_{ij1}+\bm A_{ji1}$. The random entries $1ij,i1j,ij1$ do not overlap because $i,j\ne 1$. Therefore the matrix is the sum of three Gaussian Orthogonal Ensembles, and by property of Gaussians that is equivalent to $\sqrt{3}$ times a Gaussian Orthogonal Ensemble. Now, using the result in \cite{ledoux2007deviation}, we know for any fixed $\bm x$, $\Pr[\lambda_{max}(P_{\bm x^\perp} H(\bm x) P_{\bm x^\perp}) \le \sqrt{n}/2] \le 1-\exp(-\Omega(n^2))$. By standard covering argument (the $\epsilon$-net for $n$ dimensional vectors have size $(n/\epsilon)^{O(n)}$ which is much smaller than $\exp(-\Omega(n^2))$), we know with high probability for all $\bm x$ $\lambda_{max}(P_{\bm x^\perp} H(\bm x) P_{\bm x^\perp}) \ge \sqrt{n}/2$. The lemma follows immediately because $\lambda_{max}(H(\bm x)) \ge \lambda_{max}(P_{\bm x^\perp} H(\bm x) P_{\bm x^\perp})$.
%Use \cite{tomioka2014spectral} and \cite{ledoux2007deviation}. 
\end{proof}

Now we are ready to prove Theorem~\ref{thm:homotopypath}. To capture the properties of the homotopy path, we break it into three lemmas.

    \begin{lemma}
        When $\tau = n^{3/4}\log n$, $t \ge Cn^{-1}$ for large enough constant $C$, there exists a local maximizer $\bm x^t$ of $g_r(\bm x, t)$ such that $\|\bm x^t - \bm x^\dag\|_2 = o(1) \|\bm x^\dag\|_2$.
    \end{lemma}
    \begin{proof}
Recall according to the objective we chose, the maximizer at infinity $\bm x^\dag$ can be computed explicitly and we know $\bm x^\dag = \frac{3\tau \bm v+\bm u}{3\tau(n+2)}$. By Conjecture~\ref{assump2}, we can estimate the norm and correlation with $\bm v$:
$$
\|\bm x^\dag\|_2 = \Theta(n^{-3/4}\log^{-1} n), \quad \inner{\bm x^\dag, \bm v} = (1\pm o(1))/n.
$$

We shall first prove in the region $\mathcal{B} = \{\bm x: \|\bm x - \bm x^\dag\|_2 \le \frac 1 2 \|\bm x^\dag\|_2, \inner{\bm x, \bm v} \le 10/n\}$, the Hessian of the objective function is always negative definite. By standard analysis in convex optimization, this in particular implies two things: 1. There can be at most one local maximizer in this region; 2. If the function is $\mu$-strongly-concave ($\nabla^2 g(\bm x,t)\succeq -\mu I$), and a point $\bm x$ has $\|\nabla g(\bm x,t)\| \le \epsilon$, then there is a local maximizer within $\epsilon/\mu$. This particular implies if there is a point $\bm x$ in the interior $\mathcal{B}' = \{\bm x: \|\bm x - \bm x^\dag\|_2 \le \frac 1 4 \|\bm x^\dag\|_2, \inner{\bm x, \bm v} \le 2/n\}$ such that $\nabla g(\bm x,t)$ is very small, then there must exist a local maximizer in $\mathcal{B}$.

By Lemma~\ref{lem:smoothalternative}, we know the Hessian is equal to:
$$
\nabla^2 g_r(\bm x, t) =  -3 \tau ((\|\bm x\|_2^2 + t^2 (n+2)) \bm I - 2 \langle \bm v, \bm x \rangle \bm v \bm v^T + 2 \bm x \bm x^T) + P_{sym}[\bm A(\bm x, :, :) + \bm A(:, \bm x, :) + \bm A(:, :, \bm x)].
$$

In the region we are interested in, since $\inner{\bm v,\bm x} \le 10/n \le t^2(n+2)/2$ when $C$ is large enough, we have the first term
$$
-3\tau((\|\bm x\|_2^2 + t^2 (n+2)) \bm I - 2 \langle \bm v, \bm x \rangle \bm v \bm v^T + 2 \bm x \bm x^T) \preceq -1.5\tau t^2(n+2) \bm I.
$$

On the other hand, for the second part we know by Lemma~\ref{lem:hessianprop} $$P_{sym}[\bm A(\bm x, :, :) + \bm A(:, \bm x, :) + \bm A(:, :, \bm x)] \preceq \frac{1}{2}c^+\sqrt{n}\|\bm x^\dag\|_2 I.$$

By our choice of parameters, $ \tau t^2(n+2) = \Omega(n^{-1/4}\log n)$, and $\sqrt{n}\|\bm x^\dag\|_2 = \Theta(n^{-1/4}\log^{-1}n)$, therefore the first term dominates and we know the Hessian $\nabla^2 g_r(\bm x, t) \preceq -\tau t^2(n+2) \bm I$.

%    
%We use the second order sufficient conditions, in order to prove there is a local maximizer we need to find a point with $\bm 0$ gradient and negative definite Hessian. 

When $t$ is a large polynomial of $n$ (e.g. $t = n^{10}$), simple calculation shows the optima $\bm x^t$ is very close to $\bm x^\dag$, and we have $\bm x^t \in \mathcal{B'}$. When $C/n \le t < n^{10}$, let $t_0 = n^{10}$, select $t_1,t_2,...,t_q$ such that $t_q = t$, and $t_i,t_{i+1}$ are close enough that if $\bm x^{t_i} \in \mathcal{B'}$, by strong concavity we can get $\bm x^{t_{i+1}}$ exists and $\bm x^{t_{i+1}} \in \mathcal{B}$. We will prove $\bm x^{t_i}\in \mathcal{B}'$ by induction. The base case is already done.

Suppose $\bm x^{t_{i-1}}\in \mathcal{B}'$, we know that $\bm x^{t_i} \in \mathcal{B}$. We will use the first order condition to refine our knowledge about $\bm x^{t_i}$ and show $\bm x^{t_i} \in \mathcal{B}'$. From (\ref{eq:gradient}), we can derive the expression of stationary points,
        \begin{equation}\label{eq:stationary}
            \bm x^{t_i} = \frac{3 \tau \langle \bm v, \bm x^{t_i} \rangle^2 \bm v + t^2 (3 \tau \bm v + \bm u) + \delta(\bm x^{t_i})}{3 \tau (\|\bm x^{t_i}\|_2^2 + t^2 (n+2))}
        \end{equation}
        
Note that $\bm x^{t_i}$ is a stationary point on homotopy path, so it should satisfy Conjecture~\ref{assump2}. We also know it is in $\mathcal{B}$. 
        
        Since $t \ge Cn^{-1}$, $\|\tau \langle \bm v, \bm x^{t_i} \rangle^2 \bm v\|_2 = \Theta(n^{-\frac 5 4} \log n)$, $\|t^2 (3 \tau \bm v + \bm u)\|_2 \ge \Omega(n^{-1})$ and $\|\delta(\bm x^{t_i})\|_2 = \Theta(n^{-1}\log^{-2}n)$. Therefore, if we let $\bm w = 3 \tau \langle \bm v, \bm x^{t_i} \rangle^2 \bm v + \delta(\bm x^{t_i})$ we know  $\|\bm w\|_2 \le o(1) \|t^2 (3 \tau \bm v + \bm u)\|_2$. The middle term dominates the numerator. Moreover, $t^2 (n+2) \geq \Omega(n^{-1})$ and $\|\bm x^{t_i}\|_2^2 = \Theta(n^{-3 / 2} \log^{-2} n)$, and thus, $t^2 n$ dominates the denominator. Now we have 
        \begin{align*}
            \bm x^{t_i} &~= \frac{3 \tau \langle \bm v, \bm x^{t_i} \rangle^2 \bm v + t^2 (3 \tau \bm v + \bm u) + \delta(\bm x^{t_i})}{3 \tau (\|\bm x^{t_i}\|_2^2 + t^2 (n+2))} \\&~= \frac{t^2 (3 \tau \bm v + \bm u)+\bm w}{3 \tau t^2 (n+2)(1+\epsilon)} \\
            &~= \frac{3\tau \bm v+\bm u}{3\tau (n+2)} \cdot \frac{1}{1+\epsilon} +  \frac{\bm w}{3\tau t^2(n+2)(1+\epsilon)}. \\
            &~= \bm x^\dag + \bm x^\dag (\frac{1}{1+\epsilon} - 1) + \frac{\bm w}{3\tau t^2(n+2)(1+\epsilon)}.
        \end{align*}
        Since $\epsilon = o(1)$ and $\|w\|_2 \le o(1) \|t^2 (3 \tau \bm v + \bm u)\|_2$, we know the two additional term has norm $o(1)\|\bm x^\dag\|_2$, therefore $\bm x^{t_i}$ is very close to $\bm x^\dag$.
        
        Next we bound the correlation with $\bm v$. We know $\inner{\bm v, \bm x^{t_i}} \le 10/n$ because $\bm x^{t_i} \in \mathcal{B}$. Also, the correlation between $|\inner{\bm u,\bm v}| = O(\sqrt{n \log n})$ and $|\inner{\delta(\bm x^{t_i}), \bm v}| = O(n^{-3/2} \log^{-3/2} n) $ are negligible due to Conjecture~\ref{assump2}, therefore we have
        $$
        \inner{\bm x^{t_i}, \bm v} \le \frac{3\tau(10/n)^2 + 3\tau t^2}{3\tau t^2(n+2)(1+\epsilon)} \approx \frac{10^2 + C^2}{C^2 n} \le 2/n. 
        $$
        
        Here the inequality holds as long as $C$ is large enough. Therefore $x^{t_i} \in \mathcal{B}'$ and we finish the induction.

%        Moreover, applying similar scale analysis for the Hessian, 
%        \[
%            \nabla^2 g_r((1 + o(1))\bm x^\dag, t) = -3 \tau t^2 n \bm I + \bm A((1 + o(1))\bm x^\dag, :, :) + \bm A(:, (1 + o(1))\bm x^\dag, :) + \bm A(:, :, (1 + o(1))\bm x^\dag)
%        \]
%        while the spectral norm of $\bm A((1 + o(1))\bm x^\dag, :, :) + \bm A(:, (1 + o(1))\bm x^\dag, :) + \bm A(:, :, (1 + o(1))\bm x^\dag)$ is $\Theta(\sqrt n) \|(1 + o(1))\bm x^\dag\|_2 = \Theta(n^{-\frac 1 4}) < \tau t^2 n$. Thus, with high probability, the Hessian is a negative semi-definite at the point $\bm x^\dag$.
    \end{proof}    

Next lemma shows what happens after the phase transition, when $t$ is small.

%\Rnote{In the lemma below please choose a constant $c$ that works and stick with that. We don't need to use the same $c$ as the previous section. Just get rid of the letter $c$ here completely}
    \begin{lemma}
        When $\tau = n^{3/4}\log n$, $t = n^{-1} \varepsilon(n)$, where $\varepsilon(n) = O(\log^{-2} n)$, the local maximizers (excluding saddle points) $\bm x^t$ of $g_r(\bm x, t)$ are of the following types:
        \begin{itemize}
            \item good maximizers: $\|\bm x^t\|_2^2 = \Theta(1)$ and $\langle \bm v, \bm x^t \rangle = \Theta(1)$;
            \item bad maximizers: $\|\bm x^t\|_2^2 = \Theta(n^{-\frac 1 2} \log^{-2} n)$ and $\langle \bm v, \bm x^t \rangle \leq O(n^{-\frac 1 2} \log^{-2} n)$;
        \end{itemize}
    \end{lemma}
    
    \begin{proof}
    
Now we use the second order necessary conditions. For all local maximizer, their gradient should be $\bm 0$ and their Hessian should be negative semidefinite.
    
        First, from (\ref{eq:stationary}), we can compute the inner product between $\bm v$ and $\bm x^t$:
        \[
            \langle \bm v, \bm x^t \rangle = \frac{3 \tau \langle \bm v, \bm x^t \rangle^2 + 3 \tau t^2 + t^2 \langle \bm u, \bm v \rangle + \langle \delta(\bm x^t), \bm v \rangle}{3 \tau (\|\bm x^t\|_2^2 + t^2 (n + 2))}
        \]
Note that $\bm x^t$ should satisfy the conditions in Conjecture~\ref{assump2}, in particular $|\inner{\delta(\bm x^t, \bm v}| \le O(1)\|\bm x^t\|_2^2$. Also, by Conjecture~\ref{assump2} we know $t^2 \langle \bm u, \bm v \rangle = t^2 O(\sqrt{n \log n}) \ll 3 \tau t^2$, so it is negligible in scale analysis. Therefore,
        \begin{equation}\label{eq:inner}
            \langle \bm v, \bm x^t \rangle = \frac{3 \tau \langle \bm v, \bm x^t \rangle^2 + (3 \pm o(1)) \tau t^2 \pm O(\sqrt{\log n}) \|\bm x^t\|_2^2}{3 \tau (\|\bm x^t\|_2^2 + t^2 n)}
        \end{equation}
        From (\ref{eq:stationary}), we can also compute the square of the norm of $\bm x$:
        \begin{align*}
            \|\bm x^t\|_2^2 = \frac{9 \tau^2 \langle \bm v, \bm x^t \rangle^4 + t^4 \|3 \tau \bm v + \bm u\|_2^2 + \|\delta(\bm x^t)\|_2^2 + \eta(\bm x^t)}{9 \tau^2 (\|\bm x^t\|_2^2 + t^2 (n + 2))^2} 
        \end{align*}    
        where the cross term $\eta(\bm x^t)$
        \[
            \eta(\bm x^t) = 6 \tau \langle \bm v, \bm x^t \rangle^2 \langle \bm v, \delta(\bm x^t) \rangle + 6 \tau t^2 \langle \bm v, \bm x^t \rangle^2 (3 \tau + \langle \bm v, \bm u \rangle) + 2 t^2 \langle 3 \tau \bm v + \bm u, \delta(\bm x^t) \rangle
        \]
        is negligible compared to the other terms. We again have the bound on $\|\delta(\bm x^t)\|_2$ from Conjecture~\ref{assump2} and therefore
        \begin{equation}\label{eq:norm}
            \|\bm x^t\|_2^2 = \frac{9 \tau^2 \langle \bm v, \bm x^t \rangle^4 + t^4 \Theta(n^2) + \Theta(n \log n)  \|\bm x^t\|_2^4}{9 \tau^2 (\|\bm x^t\|_2^2 + t^2 n)^2} 
        \end{equation}      
        We proceed the proof via a case analysis on the relative order between $\|\bm x^t\|_2^2$ and $t^2 n$.
        
        \paragraph{Case 1: $\|\bm x^t\|_2^2 \ge t^2 n$:} \quad \\
            
            First, recall that the Hessian at $\bm x^t$ must be a negative semidefinite. Therefore, $\tau \|\bm x^t\|_2^2$ must be larger than $\lambda_{max}(P_{sym}(\bm A(\bm x^t, :, :) + \bm A(:, \bm x^t, :) + \bm A(:, :, \bm x^t)))$. By Lemma~\ref{lem:hessianprop} we have $\tau \|\bm x^t\|_2^2 > \Theta(\sqrt n) \|\bm x^t\|_2$, which implies $\|\bm x^t\|_2 = \Omega(n^{-\frac 1 4} \log^{-1} n)$. As a result, $\Theta(n)  \|\bm x^t\|_2^4$ dominates $t^4 \Theta(n^2)$ in the nominator of (\ref{eq:norm}). Henceforth, we have
            \[
                \|\bm x^t\|_2^2 = \Theta(1)\frac{\langle \bm v, \bm x^t \rangle^4}{\|\bm x^t\|_2^4} + \Theta(n^{-\frac 1 2} \log^{-1} n)
            \]
            
We know $\|\bm x^t\|_2^2$ must be within constant factor to either $\frac{\langle \bm v, \bm x^t \rangle^4}{\|\bm x^t\|_2^4}$ or $n^{-\frac 1 2} \log^{-1} n$. These two cases are discussed below
            
            (1) If $\|\bm x^t\|_2^2 = \Theta(n^{-\frac 1 2} \log^{-1} n)$, plug it into (\ref{eq:inner}), we have
            \[
                \langle \bm v, \bm x^t \rangle = \Theta(1)\frac{\langle \bm v, \bm x^t \rangle^2}{\|\bm x^t\|_2^2} + \Theta(1)\frac{t^2}{\|\bm x^t\|_2^2} \pm \frac{O(\sqrt{\log n})}{\tau}
            \]
            Therefore, the largest possible $\langle \bm v, \bm x^t \rangle$ is $\Theta(n^{-\frac 1 2} \log^{-1} n)$. 
            
            (2) If $\|\bm x^t\|_2^3 = \Theta(1)\langle \bm v, \bm x^t \rangle^2$, plug it into (\ref{eq:inner}):
            \[
                \langle \bm v, \bm x^t \rangle = \Theta(1)\|\bm x^t\|_2^2 + \Theta(1)\frac{t^2}{\|\bm x^t\|_2^2} \pm \frac{O(\sqrt{\log n})}{\tau} = \Theta(1)\|\bm x^t\|_2^2
            \]
            Thus, we can conclude both $\|\bm x^t\|_2$ and $\langle \bm v, \bm x^t \rangle$ are bounded by absolute constants.
            
        \paragraph{Case 2: $t^2 n\ge \|\bm x^t\|_2^2$} \quad \\

We will show this case cannot happen. Recall that the Hessian at $\bm x^t$ must be a negative semidefinite. Therefore, $\tau t^2 n$ must be larger than $\lambda_{max}(P_{sym}(\bm A(\bm x^t, :, :) + \bm A(:, \bm x^t, :) + \bm A(:, :, \bm x^t)))$. By Lemma~\ref{lem:hessianprop}, we have $\tau t^2 n > \Theta(\sqrt n) \|\bm x^t\|_2$, which implies $\|\bm x^t\|_2 = t^2 \tau O(n^{1/2})$. As a result, $3 \tau t^2$ dominates $O(\sqrt{\log n})\|\bm x^t\|_2^2$ in the nominator of (\ref{eq:inner}). Henceforth, we have
            \[
                \langle \bm v, \bm x^t \rangle = C_1\frac{\langle \bm v, \bm x^t \rangle^2}{t^2 n} + C_2\frac 1 n,
            \]
            where both $C_1,C_2$ are constants within $1\pm 2/3$.
            Notice that if $\frac{\langle \bm v, \bm x^t \rangle^2}{t^2 n} \ge n^{-1}$, then $\langle \bm v, \bm x^t \rangle = \Theta(t^2 n)$, implying $\frac{\langle \bm v, \bm x^t \rangle^2}{t^2 n} = \Theta(t^2 n) =\Theta( \frac{\varepsilon^2(n)}{n}) \ll n^{-1}$. This is a contradiction, so we know, $\langle \bm v, \bm x^t \rangle$ can only be $\Theta(n^{-1})$. 
            
            Moreover, notice that $t^4 \Theta(n^2) \gg \Theta(n \log n) \|\bm x^t\|_2^4 = t^8 \tau^4 O(n^3 \log n)$.
             %and $9 \tau^2 \langle \bm v, \bm x \rangle^4 = \Theta(n^{-\frac 5 2} \log^{4c} n)$. Thus, $t^4 \Theta(n^2)$ dominates the nominator of (\ref{eq:norm}). 
             Therefore, from (\ref{eq:norm}),
            \[
                \|\bm x^t\|_2^2 = \frac {1}{t^4} \Theta(n^{-6}) + \Theta(\frac 1 {\tau^2}) \Rightarrow \|\bm x^t\|_2 = \Theta(n^{-\frac 3 4} \log^{-1} n)
            \]
            This contradicts with $\|\bm x^t\|_2 = t^2 \tau O(\sqrt n) = O(n^{-\frac 3 4} \log^{-3} n)$. There cannot be a local maximizer in this case.
    \end{proof}

Finally we show that the Hessian is correlated with the correct vector $\bm v$ near the threshold.

    \begin{lemma}
        For $t = n^{-1} \varepsilon(n)$, where $\varepsilon(n) = O(\log^{-2} n)$, let $\bm b$ be the top eigenvector of $\nabla^2 (g_r(\bm x^\dag, t))$, we know $\sin\theta(\bm b, \bm v) \le 1/\log^2 n$.
    \end{lemma}
    \begin{proof}
        Recall the formula for the Hessian (\ref{eq:hessian}), 
        \[
            \nabla^2 g_r(\bm x^\dag, t) = -3 \tau ((\|\bm x^\dag\|_2^2 + t^2 (n+2)) \bm I - 2 \langle \bm v, \bm x^\dag \rangle \bm v \bm v^T + 2 \bm x^\dag {\bm x^\dag}^T) + P_{sym}(\bm A(\bm x^\dag, :, :) + \bm A(:, \bm x^\dag, :) + \bm A(:, :, \bm x^\dag)),
        \]
        and $\bm x^\dag = \frac{\bm v} n + \frac{\bm u}{3 \tau n}$ with norm $\Theta(\frac 1 \tau)$ and correlation $\langle \bm x^\dag, \bm v \rangle = \Theta(\frac 1 n)$. Therefore, we have $\|\bm x^\dag\|_2^2 + t^2 n = O(n^{-1} \log^{-4} n)$. By Lemma~\ref{lem:hessianprop} the spectral norm of $P_{sym}(\bm A(\bm x^\dag, :, :) + \bm A(:, \bm x^\dag, :) + \bm A(:, :, \bm x^\dag))$ is $\Theta(n^{-\frac 1 4} \log^{-1} n)$. Thus, we can write the Hessian as 
        $$
        \nabla^2 g_r(\bm x^\dag, t) = 6\tau \langle \bm v, \bm x^\dag \rangle \bm v \bm v^\top + \bm E,
        $$
        where the main term $\bm v \bm v^T$ has coefficient $6 \tau \langle \bm v, \bm x^\dag \rangle = \Theta(n^{-\frac 1 4} \log n)$, and the spectral norm of $E$ is bounded by $O(n^{-1/4}\log^{-1}n)$. By Davis Kahan theorem we know $\sin\theta(\bm b, \bm v) \le O(1/\log^2 n)$, that is, the top eigenvector of the Hessian is $O(1/\log^2 n)$ close to $\bm v$. 
    \end{proof}

\end{document}